\documentclass[10pt]{article} 
\usepackage[accepted]{tmlr}

\usepackage{amsmath,amsfonts,bm}

\def\eqref#1{equation~\ref{#1}}

\def\1{\bm{1}}

\DeclareMathAlphabet{\mathsfit}{\encodingdefault}{\sfdefault}{m}{sl}
\SetMathAlphabet{\mathsfit}{bold}{\encodingdefault}{\sfdefault}{bx}{n}

\usepackage{xcolor}
\usepackage{hyperref}
\usepackage{booktabs}
\usepackage{multirow}
\usepackage[T1]{fontenc}
\usepackage{lipsum}
\usepackage{subcaption}
\usepackage{amsmath}
\usepackage{amsthm}
\usepackage{amssymb}
\usepackage{relsize}
\usepackage{wrapfig}
\usepackage{graphicx}
\usepackage[misc]{ifsym}

\renewcommand{\sectionautorefname}{\S\kern-4pt}
\renewcommand{\subsectionautorefname}{\S\kern-4pt}

\usepackage[normalem]{ulem}
\useunder{\uline}{\ul}{}

\newcommand*{\concat}{|\mskip-3mu|}
\newcommand*{\bigconcat}{\Big|\mskip-3mu\Big|}
\newcommand*{\biggconcat}{\bigg|\mskip-3mu\bigg|}

\newcommand*{\ldblbrace}{\{\mskip-5mu\{}
\newcommand*{\rdblbrace}{\}\mskip-5mu\}}
\newcommand*{\bigldblbrace}{\bigg\{\mskip-10mu\bigg\{}
\newcommand*{\bigrdblbrace}{\bigg\}\mskip-10mu\bigg\}}

\newtheorem{lemma}{Lemma}
\newtheorem{theorem}{Theorem}
\newtheorem{remark}{Remark}

\newtheorem{corollary}{Corollary}
\newtheorem{proposition}{Proposition}

\usepackage{algorithm}
\usepackage{algorithmicx}
\usepackage{algpseudocode}
\algrenewcommand\algorithmicrequire{\textbf{Input:}}
\algrenewcommand\algorithmicensure{\textbf{Output:}}
\algnewcommand{\LineComment}[1]{\State \(\triangleright\) #1}

\algdef{SE}[DOWHILE]{Do}{DoWhile}{\algorithmicdo}[1]{\algorithmicwhile\ #1}%

\definecolor{plum}{RGB}{221,160,221}

\usepackage[bottom]{footmisc}

\usepackage{xargs}

\newcommandx{\Replace}[1]{\textcolor[HTML]{FF0000}{#1}}
\newcommandx{\Previous}[1]{\textcolor[HTML]{880000}{#1}}
\newcommandx{\Addition}[1]{\textcolor[HTML]{008800}{#1}}

\def\useNotes{0}
\def\cameraReady{1}

\ifx\useNotes\undefined
\else
  \if\useNotes1
    \usepackage[prependcaption,textsize=tiny]{todonotes}
    \newcommandx{\VGn}[2][1=]{\todo[linecolor=blue,backgroundcolor=blue!25,bordercolor=blue,#1]{#2}}
    \newcommandx{\NAn}[2][1=]{\todo[linecolor=red,backgroundcolor=red!25,bordercolor=red,#1]{#2}}
    \newcommandx{\AKn}[2][1=]{\todo[linecolor=green,backgroundcolor=green!25,bordercolor=green,#1]{#2}}
    \newcommandx{\DNn}[2][1=]{\todo[linecolor=gray,backgroundcolor=gray!25,bordercolor=gray,#1]{OK: #2}}
    \newcommandx{\VGnOk}[2][1=]{\todo[linecolor=blue,backgroundcolor=blue!10,bordercolor=blue,#1]{#2}}
    \newcommandx{\NAnOk}[2][1=]{\todo[linecolor=red,backgroundcolor=red!10,bordercolor=red,#1]{#2}}
    \newcommandx{\AKnOk}[2][1=]{\todo[linecolor=green,backgroundcolor=green!10,bordercolor=green,#1]{#2}}
  \else
    \newcommandx{\VGn}[2][1=]{}
    \newcommandx{\NAn}[2][1=]{}
    \newcommandx{\AKn}[2][1=]{}
    \newcommandx{\DNn}[2][1=]{}
    \newcommandx{\VGnOk}[2][1=]{}
    \newcommandx{\NAnOk}[2][1=]{}
    \newcommandx{\AKnOk}[2][1=]{}
  \fi
\fi

\title{Improving Subgraph-GNNs via Edge-Level\\Ego-Network Encodings}

\author{\name Nurudin Alvarez-Gonzalez \email nuralgon@gmail.com \\
      Universitat Pompeu Fabra
      \AND
      \name Andreas Kaltenbrunner \email kaltenbrunner@gmail.com \\
      \addr Universitat Oberta de Catalunya \\
      \addr ISI Foundation Turin
      \AND
      \name Vicen\c{c} G\'omez \email vicen.gomez@upf.edu\\
      Universitat Pompeu Fabra}

\def\month{04}  
\def\year{2024} 
\def\openreview{\url{https://openreview.net/forum?id=N0Sc0KY0AH}} 

\begin{document}

\maketitle

\begin{abstract}
We present a novel edge-level ego-network encoding for learning on graphs 
that can boost Message Passing Graph Neural Networks (MP-GNNs) by providing additional node and edge features or extending message-passing formats. The proposed encoding is sufficient to distinguish Strongly Regular Graphs, a family of challenging 3-WL equivalent graphs. We show theoretically that such encoding is more expressive than node-based sub-graph MP-GNNs. In an empirical evaluation on four benchmarks with 10 graph datasets, our results match or improve previous baselines on expressivity, graph classification, graph regression, and proximity tasks---while reducing memory usage by 18.1x in certain real-world settings. 
\end{abstract}\section{Introduction}
\label{sec:EleneIntroduction}
Neural graph architectures are the current standard for learning from graph data.
These methods automatically learn representations of nodes, edges, or graphs in a data-driven and end-to-end way. 
Message Passing Graph Neural Networks (MP-GNNs) are the most common model for learning on graphs.
MP-GNNs process the input graph (data) with a computational graph (model) to learn useful representations through message passing between direct neighbors in the input graph. This framework facilitates the theoretical analysis of MP-GNNs (e.g., in terms of their expressive power~\citep{xu2018how}, or characterizing issues such as \emph{over-squashing}~\citep{alon2021on} or \emph{over-smoothing}~\citep{oono2019graph}).

The idea of decoupling the input graph from the computational graph is the basis of leading-edge learning approaches, such as Sub-graph GNNs~\citep{zhao2022,abboud2022spnn,frasca2022understanding,mitton2023subgraph},
perturbation methods~\citep{papp2021dropgnn,dwivedi2022graph}, or Graph Transformers~\citep{Yun2019,ying2021do,rampasek2022GPS,kim2022pure}. These methods extend the message-passing mechanism to more general structures induced by the graph, beyond direct neighbors---for example between the nearest neighbours of a node at a given depth (ego-networks). This flexibility extends the expressive power of MP-GNNs, 
at the cost of an increased computational footprint and a departure of the inductive bias contained in the input graph, which must be learned again.

In this work, we present an alternative to previous \emph{pure learning} approaches. Rather than learning on sub-graphs, we introduce a systematic procedure to generate a pool of structural features (an encoding) which can subsequently be integrated into MP-GNNs.
Similar approaches have been proposed recently~\citep{bouritsas2023,alvarez-gonzalez2022beyond}. 
Crucially, our proposed features capture information at the edge-level, including signals contained in the two ego-networks of adjacent nodes in the input graph. We call this encoding \textsc{Elene}, for \textbf{E}dge-\textbf{L}evel \textbf{E}go-\textbf{N}etwork \textbf{E}ncodings. The benefits of such a representation are diverse: The encodings are interpretable and amenable for theoretical analysis, they are efficiently computable as a pre-processing step, and finally, they reach comparable performance with state-of-the-art learning methods.

As an illustrative example, consider
\textbf{S}trongly \textbf{R}egular \textbf{G}raphs (\texttt{SRG}s). 
They are known to be \emph{indistinguishable} by node-based sub-graph GNNs~\citep{balcilar2021breaking,morris2023,pmlr-v162-papp22a,zhao2022,frasca2022understanding}, 
 as exemplified by the non-isomorphic $4\times 4$ Rook and Shrikhande graphs  in~\autoref{fig:3WLGraphs}.
We theoretically show that \textsc{Elene} is as expressive as node-only sub-graph GNNs, and expressive enough to differentiate certain classes of \texttt{SRG}s like those in~\autoref{fig:3WLGraphs}.

\begin{figure}[!htt]
  \centering
  \begin{subfigure}[b]{0.325\linewidth}
    \centering
    \includegraphics[width=\linewidth]{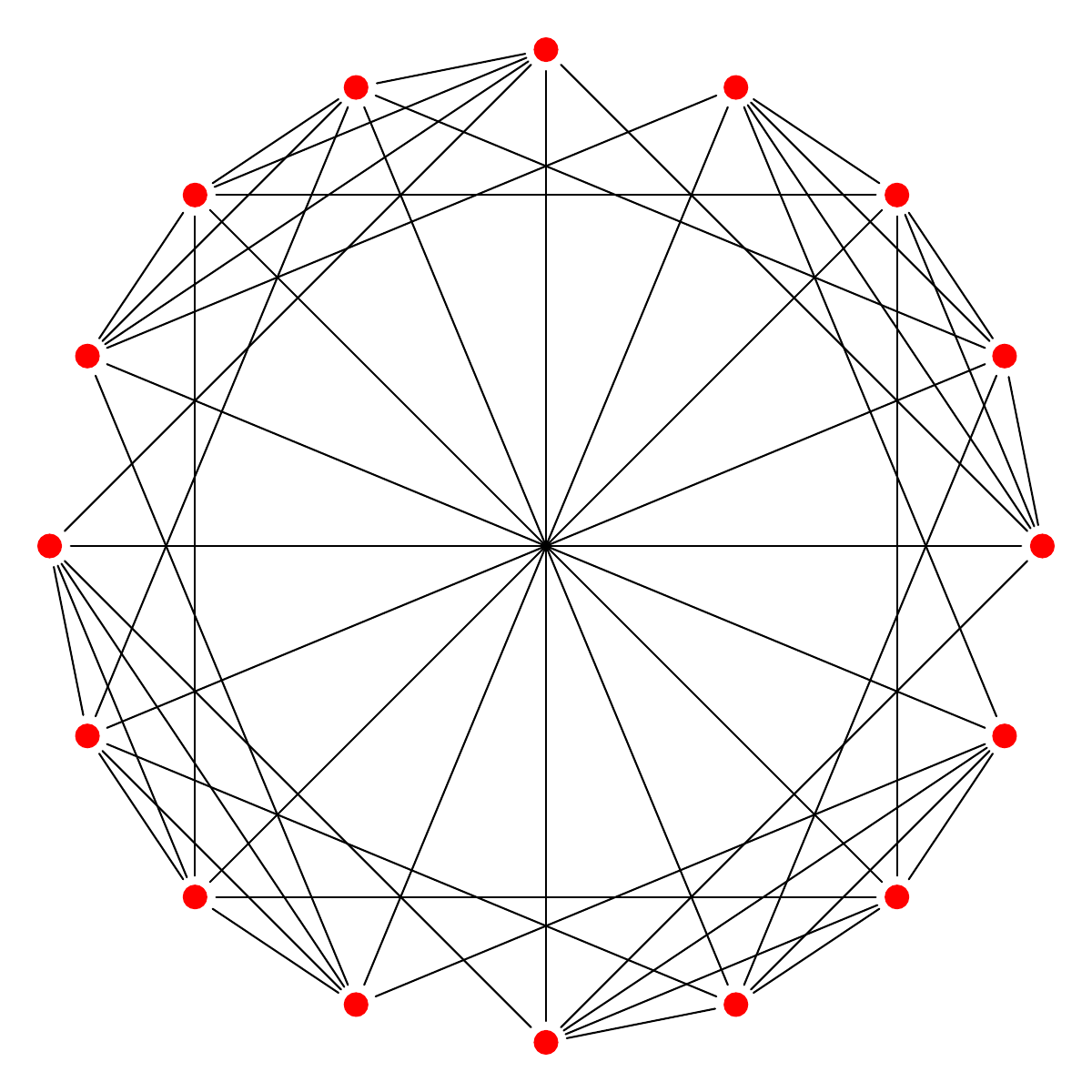}
    \caption[a]{$4\times4$ Rook Graph.}
  \end{subfigure}
  \hspace{0.1\linewidth}
  \begin{subfigure}[b]{0.325\linewidth}
    \centering
    \includegraphics[width=\linewidth]{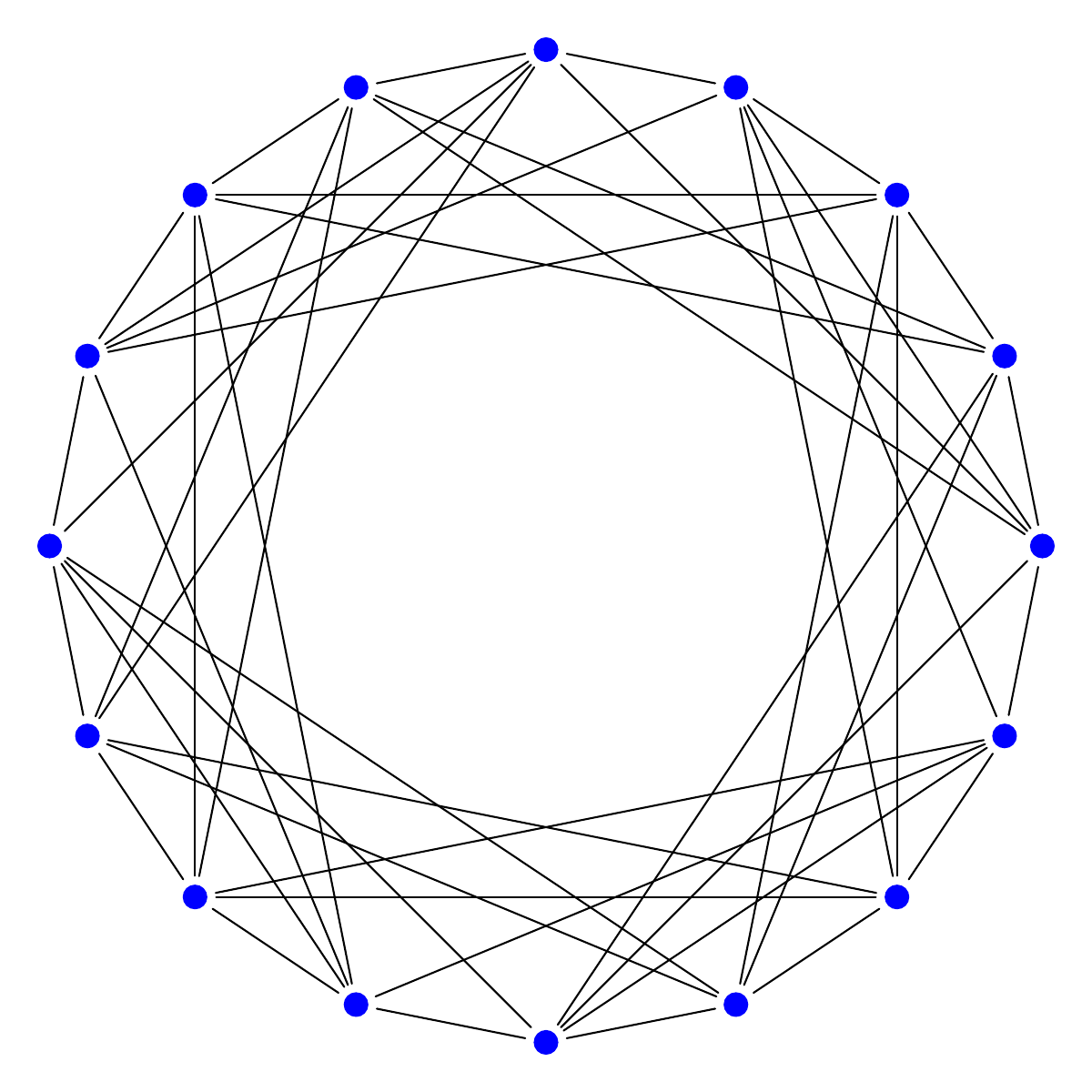}
    \caption[b]{Shrikhande Graph.}
  \end{subfigure}
  \caption{Expressive power is typically analyzed in terms of the families of non-isomorphic graphs \emph{that models fail to distinguish}: $4\times 4$ Rook (a) and Shrikhande (b) graphs are indistinguishable by node-only sub-graph GNNs~\citep{frasca2022understanding}.} 
  \label{fig:3WLGraphs}
\end{figure}
Another example of a challenging benchmark is the $h$-Proximity task (shown in~\autoref{fig:hProximity}), which requires the ability to capture graph properties that depend both on the graph structure (shortest path distances) and node attributes (colors)~\citep{abboud2022spnn}.
In this case, an enriched (learnable) MP-GNN with \textsc{Elene} features---called \textsc{Elene-L}---outperforms current baselines. In real-world benchmarks, \textsc{Elene-L} matches the performance of state-of-the-art learning methods at significantly lower memory costs, as we show experimentally in \autoref{sec:EleneExperiments}.

The paper is organized as follows. \autoref{sec:EleneDef} defines and motivates \textsc{Elene}. \autoref{sec:EleneLearning} introduces \textsc{Elene-L}. \autoref{sec:EleneRelWork} describes related work and~\autoref{sec:EleneExpressivity} analyzes expressivity. Finally, \autoref{sec:EleneExperiments} evaluates our methods in four benchmarks and \autoref{sec:EleneConclusions} summarizes our results.
\begin{figure}[!ht]
  \centering
  \begin{subfigure}[b]{0.35\linewidth}
    \centering
    \includegraphics[width=.8\linewidth]{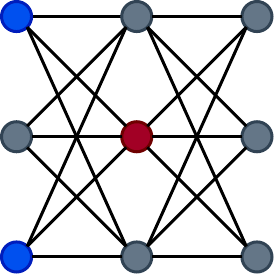}
    \caption[a]{Positive graph.}
  \end{subfigure}
  \hspace{0.1\linewidth}
  \begin{subfigure}[b]{0.35\linewidth}
    \centering
    \includegraphics[width=.8\linewidth]{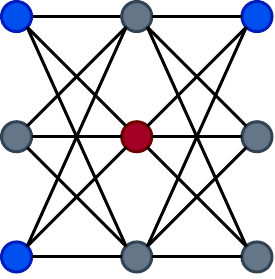}
    \caption[b]{Negative graph.}
  \end{subfigure}
  \caption{$h$-Proximity binary classification task---A pair of positive (a) and negative (b) 1-Proximity graph examples. An $h$-Proximity graph is positive if all red nodes have at most 2 blue neighbors up to distance $h$, and negative otherwise.} 
  \label{fig:hProximity}
\end{figure}

\section{Notation and Definitions}
\label{sec:EleneNotation}

In this work, $G = (V, E)$ denotes a graph with $n = |V|$ and $m = |E|$. $l_G(u, v)$ is the shortest path length between $u, v \in V$ in $G$. $d_{G}(v)$ is the degree of $v$ in $G$ and we use $d_{\texttt{max}}$ for the maximum degree over all nodes in $G$. Double brackets $\ldblbrace\cdot\rdblbrace$ denote multi-sets while $\bigcup$ and $\bigcap$, respectively, indicate set and multi-set union and intersection. We use short-hand $x^r$ notation to signify $x$ is contained $r$ times, where $y = \ldblbrace x^r \rdblbrace$ reads as ``$x$ appears $r$ times in $y$''. 

We use $\mathcal{S}^{k}_v = (\mathcal{V}^{k}_v, \mathcal{E}^{k}_v) \subseteq G$ for the $k$-depth induced ego-network sub-graph of $G$ centered on $v$ (abbreviated $\mathcal{S}$ in equations). 
We denote the maximum degree over all nodes in $\mathcal{S}^{k}_v$ by $d^k_{\texttt{max}}$.
Likewise, we use $\mathcal{S}^{k}_{\langle u, v\rangle} = (\mathcal{V}^{k}_u \bigcap \mathcal{V}^{k}_v, \mathcal{E}^{k}_u \bigcap \mathcal{E}^{k}_v)$ to denote the intersection of ego-networks across edge $\langle u, v \rangle$. Feature vectors are shown in \textbf{bold}, as $\mathbf{x}_v$ for node $v$, $\mathbf{x}_{\langle u, v \rangle}$ for edge $\langle u, v \rangle$, we denote vector concatenation by $\concat$, and the Hadamard product by $\odot$. Finally, we represent a learnable embedding of a discrete input, e.g., degree or distance signals as $\texttt{Emb}(\cdot)$, and a learnable weight matrix as $\textbf{W}$.

\section{Defining ELENE}\label{sec:EleneDef}

In this section, we first present the proposed edge-level encodings and then illustrate their expressive power.

\subsection{Constructing an Edge-Level Ego-Network Encoding}
The main idea behind \textsc{Elene} encodings is to capture higher-order interactions that go beyond the node-centric perspective used by MP-GNNs.
We look at the structure resulting not only from the ego-network of every node, but also from the combination of two ego-networks of adjacent nodes in the input graph, and design a pool of features based on that structure.

Consider the $k$-depth ($k > 1$) ego-network $\mathcal{S}^{k}_v$ surrounding node $v$. We may ask: \emph{how many edges of a neighbor $u$ of $v$ reach nodes that are 1-hop closer to $v$, at the same distance as $u$, or 1-hop farther from $v$?} The proposed \textsc{Elene} encodings elaborate on this idea to capture interactions between nodes and edges in ego-network sub-graphs.

More formally, consider a node $u$ contained in $\mathcal{S}^{k}_v$ and 
let $d_{\mathcal{S}}(u|v)$ count the edges from $u$ to nodes at a distance $l_\mathcal{S}(u,v)+p$ of $v$, $p 
\in\{-1, 0, +1\}$:
\begin{align*}
    d_{\mathcal{S}}^{(p)}&(u|v) = \Big|(u, w) \in \mathcal{E}^{k}_v,\ \forall w \in \mathcal{V}^{k}_v : l_{\mathcal{S}}(v, w) = l_{\mathcal{S}}(u, v) + p\Big|.
\end{align*}

The degree of node $u$ decomposes as the sum of these \emph{relative} degrees corresponding to these three different subsets of neighbors of $u$:
\begin{align*}
d_{\mathcal{S}}(u) = d_{\mathcal{S}}^{(\texttt{-}1)}(u|v) + d_{\mathcal{S}}^{(0)}(u|v) + d_{\mathcal{S}}^{(\texttt{+}1)}(u|v).
\end{align*}

\autoref{fig:ELENEDegrees} (left) shows an example graph, with all nodes labeled with their degree and colored according to the distance to the root node of the ego network, in this case, \textcolor{green!50!black}{the node in green}.
The plot on the right shows a degree triplet for each node, which counts the \emph{relative} degrees, or edges closer and further to the \textcolor{green!50!black}{root} ($1^{\text{st}}$ and $3^{\text{rd}}$ components), together with the individual degree ($2^{\text{nd}}$ component).

\begin{figure}[!t]
    \centering
    \includegraphics[width=0.8\linewidth]{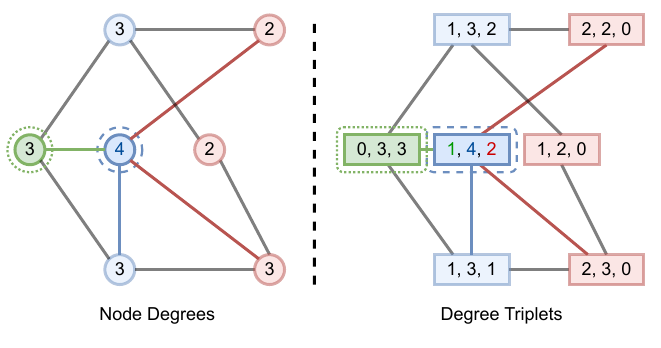}
    \caption{Example graph (right) and corresponding (left) degree triplets for nodes in the 2-hop ego-network rooted on \textcolor{green!50!black}{the green node}. 
    The \textcolor{blue!75!black}{dashed blue node} has one edge to the \textcolor{green!50!black}{0-hop root ($d^{(\texttt{-}1)}_{\mathcal{S}} = 1$)}, a \textcolor{blue!75!black}{degree of 4}, and two edges \textcolor{red!80!black}{2-hops from the root ($d^{(\texttt{+}1)}_{\mathcal{S}} = 2$, red)}, so its degree triplet is (\textcolor{green!50!black}{1}, \textcolor{blue!75!black}{4}, \textcolor{red!80!black}{2}).} 
    \label{fig:ELENEDegrees}
\end{figure}
Leveraging relative degrees yields \textsc{Elene}, an ego-network encoding as a multi-set of quadruplets counting all instances of distance and degree triplets in sub-graph $\mathcal{S}$:
\begin{align}\label{eq:ELENEEncRaw}
    \mskip-10mu e^{k}_v = \bigldblbrace\mskip-5mu\Big(l_{\mathcal{S}}(u, v), d^{(\texttt{-}1)}_{\mathcal{S}}(u|v), d_{\mathcal{S}}(u), d^{(\texttt{+}1)}_{\mathcal{S}}(u|v)\Big) 
        \Big|\forall u\mskip-5mu\in\mskip-5mu\mathcal{V}^{k}_v\bigrdblbrace.
\end{align} 
We can construct an Edge (\textbf{ED}) Centric encoding analogous to the Node (\textbf{ND}) Centric encoding of~\autoref{eq:ELENEEncRaw} by also encoding edge-wise sub-graph intersections for edge $\langle u, v \rangle$ as $e^{k}_{\langle u, v \rangle}$ and counting quadruplets across $\mathcal{S}^{k}_{\langle u, v\rangle}$ with distances to both $u$ and $v$.
Using \textbf{ED} or \textbf{ND} encodings leads to different expressive power, as we show formally in~\autoref{sec:EleneExpressivity}. In both cases, \autoref{sec:ELENEAppPermInvariant} shows that \textsc{Elene} encodings are permutation invariant at the node level and equivariant at the graph level.

\subsection{Illustrating ELENE}\label{subsec:EleneIntuition}


To illustrate \textsc{Elene}, we focus on  Strongly Regular graphs. An $n$-vertex graph is $d$-regular if all $n$ nodes have degree $d$, i.e., $\forall\ v \in V, d_{G}(v) = d$. An $n$-vertex $d$-regular graph is said to be Strongly Regular if there exists $\lambda, \mu \in \mathbb{N}$ such that every two adjacent nodes have $\lambda$ neighbors in common, and every two non-adjacent nodes have $\mu$ neighbors in common. We denote Strongly Regular Graphs as $\texttt{SRG}(n, d, \lambda, \mu)$. Strongly Regular graphs with equal parameters are \emph{indistinguishable} by the 1-WL~\citep{weisfeilerlemanl1968wl} test---a classic graph algorithm known to distinguish graphs that are not isomorphic with high probability~\citep{babai1979}---and its more powerful $k=3$-WL variant~\citep{arvind2020,balcilar2021breaking}---whose ability to distinguish graphs has been shown to be the expressivity upper-bound for node-only Sub-graph GNNs~\citep{bevilacqua2022equivariant,frasca2022understanding,zhao2022}. A natural question follows: \emph{what structural information is sufficient to distinguish \texttt{SRG}s?}

\begin{figure}[!ht]
  \centering
  \begin{subfigure}[b]{0.4\linewidth}
    \centering
    \includegraphics[width=\linewidth]{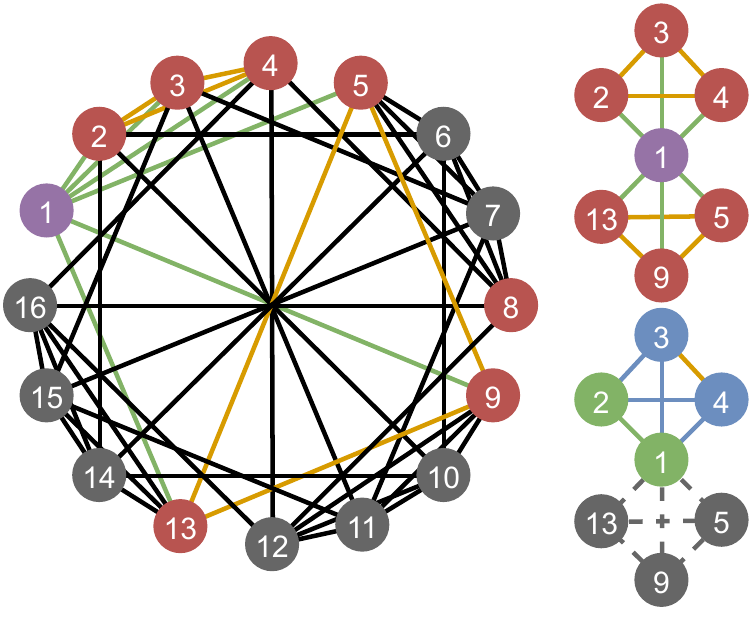}
    \caption[a]{$4\times4$ Rook Graph.}
  \end{subfigure}
  \begin{subfigure}[b]{0.4\linewidth}
    \centering
    \includegraphics[width=\linewidth]{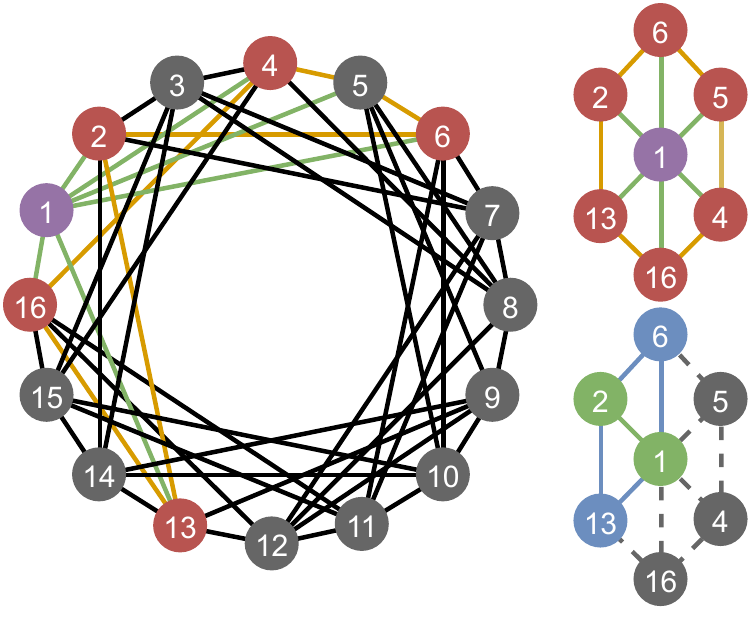}
    \caption[b]{Shrikhande Graph.}
  \end{subfigure}
  \caption[$4\times 4$ Rook and Shrikhande graphs, encoded by \textsc{Elene} (\textbf{ND}) and (\textbf{ED}).]{The $4\times 4$ Rook (a) and Shrikhande (b) graphs are indistinguishable by 3-WL as \texttt{SRG}s with parameters $\texttt{SRG}(16, 6, 2, 2)$~\citep{arvind2020,frasca2022understanding}. \textsc{Elene} (\textbf{ND}, top sub-graphs) is also unable to distinguish the graphs, while \textsc{Elene} (\textbf{ED}, bottom sub-graphs) counts different numbers of edges.} 
  \label{fig:3WLGraphsEleneEdgeDistinguishable}
\end{figure}

In~\autoref{fig:3WLGraphsEleneEdgeDistinguishable}, we show the 1-depth ego-networks $\mathcal{S}^{k=1}_{v_1}$ for the \textcolor{plum!55!black}{purple vertices labeled with `1'} with its \textcolor{red!80!black}{1-hop neighbors colored in red} (top smaller sub-graphs)\footnote{Note that for these two graphs, any node will have matching ego-networks regardless of their label---see proof for~\autoref{th:EleneMoreExpressive}.}. 
We represent both graphs in terms of $n=16$ equal sub-graphs (one per node), analyzing whether the sub-graph pairs can be distinguished. Both sub-graphs have the same number of nodes ($7$), edges ($12$), and matching degree multisets $\ldblbrace 3^6, 6^1 \rdblbrace$. Furthermore, by coloring the edges as connected to the ego-network root (in \textcolor{green!50!black}{green}) or connecting adjacent neighbors of the root (in \textcolor{orange!70!black}{orange}), the count of edge colors also matches. The Node-Centric (\textbf{ND}) \textsc{Elene} encoding, as shown in~\autoref{eq:ELENEEncRaw}, corresponds to such a coloring and is thus unable to distinguish the pair of graphs. In~\autoref{sec:EleneExpressivity}, we formally prove this upper bound for \textsc{Elene} (\textbf{ND}), which coincides with the expressive power of node-based Sub-graph GNNs.

In contrast, if we consider \textcolor{blue!75!black}{the 1-hop common neighbors (in blue)} of \textcolor{green!50!black}{adjacent nodes labeled `1' and `2'}, the intersecting sub-graphs are distinguishable (bottom smaller sub-graphs). Indeed, the number of edges differs between the $4\times 4$ Rook graph ($6$ edges) and the Shrikhande's graph ($5$ edges). This corresponds to Edge-Centric \textsc{Elene} (\textbf{ED}), and illustrates it has more expressive power than the Node-Centric (\textbf{ND}) one. 

\subsection{Computational Complexity}\label{subsec:ELENEAlgorithmicAnalysis}

The \textsc{Elene} encoding for a single node $v$ requires traversing all the edges in the ego-network $\mathcal{E}^{k}_v$. This can be computed via Breadth-First Search (BFS) bounded by depth $k$, which has worst-case complexity~$\mathcal{O}(d_{\texttt{max}}^k)$. If~$k$ is greater than the diameter in the graph, \textsc{Elene} must traverse all $m$ edges for the $n$ ego-networks with each node as the root. Encoding the entire graph thus has time complexity $\mathcal{O}(n
\cdot \min\{m,d_{\texttt{max}}^k\})$. 
Note that the more expressive edge-centric implementation requires executing the BFS from both nodes alongside an edge, with asymptotically no additional cost.

\textsc{Elene} is best suited for sparse graphs, where $d_{\texttt{max}}\ll n$. For fully connected graphs, $m = |\mathcal{E}^{k}_v| = n \cdot (n - 1) / 2$ which results in time complexity $\mathcal{O}(n^3)$, matching the computational worst-case complexity of GNN-AK~\cite{zhao2022}, NGNNs~\cite{zhang2021}, SUN~\cite{frasca2022understanding}, ESAN~\cite{bevilacqua2022equivariant}, or SPEN~\cite{mitton2023subgraph}.

In terms of memory, \textsc{Elene} encodings require a sparse $3 \cdot (k + 1) \cdot (d_{\texttt{max}} + 1)$-component vector for each node $v \in V$ to represent the multi-set of quadruplets in~\autoref{eq:ELENEEncRaw}. Accordingly, each entry holds the count of observed relative degrees at each distance from $v$. In \autoref{sec:EleneBFS}, we describe a BFS implementation that produces a mapping of each \textsc{Elene} encoding quadruplet to its frequency, and can be parallelized over $p$ processors yielding $\mathcal{O}(n \cdot \min\{m, d_{\texttt{max}}^k\} / p)$ time complexity.

\section{Learning with ELENE: ELENE-L}\label{sec:EleneLearning}

We now introduce two approaches for leveraging \textsc{Elene} encodings in practical learning settings---a simple concatenation of \textsc{Elene}  over network attributes, and a fully learnable variant called \textsc{Elene-L} that updates node and edge representations during the learning process. The first approach represents \textsc{Elene} multi-as sparse vectors containing frequencies of each quadruple $q$---which can be attributes concatenated to $\textbf{x}_v$ or $\textbf{x}_{\langle u, v\rangle}$ if processing $e_v^k$ or $e^{k}_{\langle u, v \rangle}$:
\begin{align}\label{eq:ELENESparseVector}
    \textsc{Elene}_{\texttt{vec}}^{k}(v)_{i} = \bigg|\bigldblbrace q \in e^{k}_v
         \Big| f(q) = i\bigrdblbrace\bigg|.
\end{align}
where $f(q)$ is an indexing function mapping each unique quadruplet to an index in the sparse vector. 

The concatenation approach is the most memory efficient, using only as much memory as the encodings themselves, and can be computed once and reused during training or inference. Furthermore, this approach is directly applicable to any downstream learning model e.g., an MP-GNN, without changing its architecture. 

Certain tasks, however, require structural information \emph{within the sub-graph} to be combined with node or edge \emph{attributes} during learning. One example are $h$-Proximity tasks, which require joint representations that integrate the \textsc{Elene} encodings with attributes and structure. 

\textsc{Elene-L} addresses the limitations of concatenating \textsc{Elene} with node and edge attributes by learning over both \emph{structures} and \emph{attributes} at once. \textsc{Elene-L} learns non-linear functions ($\Phi$, e.g. a Dense Neural Network---DNN) to represent nodes, edges, and Node or Edge-Centric sub-graphs by
representing $u$ in $\mathcal{S}^{k}_v$ via a learnable function $\Phi_{\texttt{nd}}$\footnote{We compress notation by using $\Phi^{t}$ for the output of $\Phi$ at step $t$.}:
\begin{align}\label{eq:ELENE-L-NodeRep}
    \Phi_{\texttt{nd}}^{t}(u|v) = \Phi_{\texttt{nd}}\Big(\textbf{x}_v^{t} \bigconcat \textbf{x}_u^{t} \bigconcat \texttt{Emb}(u|v)\Big),
\end{align}
where $\textbf{x}_v^{t}$ and $\textbf{x}_u^{t}$ are features of $u$ and $v$ at time-step $t$ (i.e. after $t$ layers, such that $t=0$ are `input' features), and $\texttt{Emb}(u|v)$ is a learnable embedding of \textsc{Elene} encodings which we describe in~\autoref{subsec:EleneEmbeddings}. As with \textsc{Elene}, we produce a learnable representation of an edge $\langle u, w\rangle$ in $\mathcal{S}^{k}_v$ via a learnable $\Phi_{\texttt{ed}}$:
\begin{align*}
    \Phi_{\texttt{ed}}^{t}(u, w|v) = \Phi_{\texttt{ed}}\Big(\textbf{x}_v^{t} \bigconcat \textbf{x}_{\langle u, w\rangle}^{t} \bigconcat \textbf{x}_u^{t} \odot \textbf{x}_w^{t} \bigconcat \texttt{Emb}(u, w|v)\Big).
\end{align*}
The representation of the Node-Centric ego-network root at time $t$ is a learnable $\Phi_{\texttt{ND}}$ applied over the aggregation of every node and edge in the sub-graph given a pooling function ($\sum$):
\begin{align}\label{eq:ELENE-L-NodeMessage}
    \Phi_{\texttt{ND}}^{t}(v) \mskip3mu=\mskip3mu \Phi_{\texttt{ND}}\Bigg(\textbf{x}_v^{t}\biggconcat \sum_{u}^{\mathcal{V}^{k}_v} \Phi_{\texttt{nd}}^{t}(u|v)\biggconcat\sum_{\langle u, w\rangle}^{\mathcal{E}^{k}_v} \Phi_{\texttt{ed}}^{t}(u, w|v)\Bigg).
\end{align}
Similarly, the Edge-Centric representation for edge $\langle u, v\rangle$ at time $t$ is a learnable $\Phi_{\texttt{ED}}$ consuming the aggregation over ego-networks containing the edge, as shown in~\autoref{fig:EleneLEdgeDiagram}:
\begin{align}\label{eq:ELENE-L-EdgeMessage}
    \Phi_{\texttt{ED}}^{t}(u, v) = \Phi_{\texttt{ED}}\bigg(\mathlarger{\sum}_{w}^{\mathcal{V}^{k}_{\langle u, v\rangle}} \Phi_{\texttt{ed}}^{t}(u, v|w)\bigg).
\end{align}
\begin{figure}[!b]
    \centering
    \includegraphics[width=1.0\linewidth]{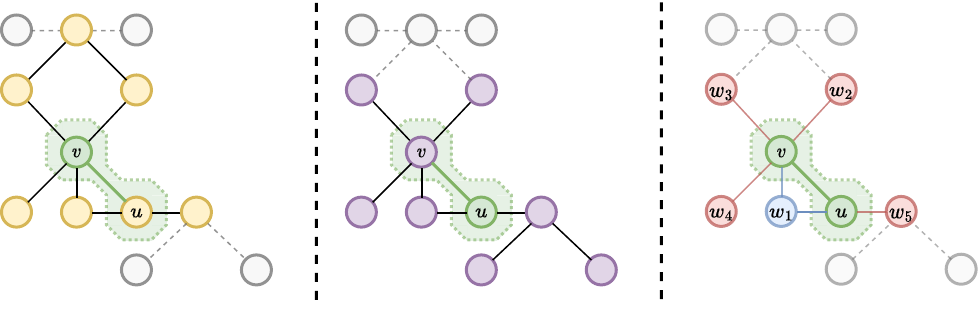}
    \caption{$k$-depth ego-network intersection following~\autoref{eq:ELENE-L-EdgeMessage} for the \textcolor{green!60!black}{green edge}. The ego-networks of \textcolor{orange!70!black}{$u$} and \textcolor{plum!55!black}{$v$} (\textcolor{orange!70!black}{yellow (left)} and \textcolor{plum!55!black}{purple (center)} respectively), intersect on five nodes around \textcolor{green!60!black}{$\langle u, v \rangle$ (dotted, right)}. We show $\mathcal{V}^{k=2}_{\langle u, v\rangle} = \{\textcolor{green!60!black}{u, v}, \textcolor{blue!70!black}{w_1}, \textcolor{red!60!black}{w_2, w_3, w_4, w_5}\}$ (right), indicating nodes reachable in \textcolor{green!60!black}{0 or 1-hops}, \textcolor{blue!70!black}{exactly 1-hop} or \textcolor{red!60!black}{1 or 2-hops} from $u$ and $v$.}
    \label{fig:EleneLEdgeDiagram}
\end{figure}
Node and edge representations at $t+1$ update via a learnable parameter $\gamma$ gating the flow of \textsc{Elene} updates:
\begin{align}\label{eq:ELENE-L}
    \textbf{x}_v^{t + 1} = \textbf{x}_v^{t} + \gamma_{\texttt{ND}} \cdot \Phi_{\texttt{ND}}^{t}(v).
\end{align}
We follow the same update-rule at the edge level:
\begin{align}\label{eq:ELENE-L-Edge}
    \textbf{x}_{\langle u, w\rangle}^{t + 1} = \textbf{x}_{\langle u, w\rangle}^{t} + \gamma_{\texttt{ED}} \cdot \Phi_{\texttt{ED}}^{t}(u, w).
\end{align}
We may use $\textbf{x}_v^{t + 1}$ and $\textbf{x}_{\langle u, w\rangle}^{t + 1}$ directly in the downstream task, or as inputs into an MP-GNN layer during learning---boosting its expressivity. We follow the latter approach in this work.

\subsection{Defining ELENE-L Embeddings}\label{subsec:EleneEmbeddings}

The representations in~\autoref{eq:ELENE-L} and~\autoref{eq:ELENE-L-Edge} leverage attributes and \textsc{Elene} encodings through $\texttt{Emb}(u|v)$ and $\texttt{Emb}(u, w|v)$. To define \textsc{Elene-L} embeddings, three hyper-parameters determine the shapes of embedding matrices: $\omega$, the length of the embedding vectors; $\rho$, the max. degree to be encoded (by default, $\rho = d_{\texttt{max}}$); and $k$, the maximum distance to be encoded (i.e., the ego-network depth). For the quadruplet of $u$, we abbreviate: $$q_u = (\textcolor{green!60!black}{l_u}, \textcolor{red!60!black}{d^1_u}, \textcolor{blue!70!black}{d^2_u}, \textcolor{red!60!black}{d^3_u}) = \big(l_{\mathcal{S}}(u, v), d^{(\texttt{-}1)}_{\mathcal{S}}(u|v), d_{\mathcal{S}}(u), d^{(\texttt{+}1)}_{\mathcal{S}}(u|v)\big)$$ as defined in~\autoref{eq:ELENEEncRaw} and \emph{jointly} embed distance and relative degrees. The embedding $\texttt{Emb}(u|v)$ of $u$ in $\mathcal{S}^{k}_v$ is given by:
\begin{align}\label{eq:EleneNodeEmbeddings}
    &\texttt{Emb}\Big[\textcolor{green!60!black}{l_u}, \textcolor{red!60!black}{d^1_u}, \textcolor{blue!70!black}{d^2_u}, \textcolor{red!60!black}{d^3_u}\Big] = \Big(\textbf{W}_{(\textcolor{green!60!black}{l_u},\textcolor{red!60!black}{d^1_u})}^{1,\texttt{nd}} \bigconcat \textbf{W}_{(\textcolor{green!60!black}{l_u},\textcolor{blue!70!black}{d^2_u})}^{2,\texttt{nd}} \bigconcat \textbf{W}_{(\textcolor{green!60!black}{l_u},\textcolor{red!60!black}{d^3_u})}^{3,\texttt{nd}}\Big),
\end{align}
where $\textbf{W}^{1,\texttt{nd}}$, $\textbf{W}^{2,\texttt{nd}}$, and $\textbf{W}^{3,\texttt{nd}} \in \mathbb{R}^{S \times \omega}$ are three node embedding matrices with $S = (\rho + 1) \cdot (k + 1)$ entries---one for each distance and relative degree pair. A visual representation of the attributes and \textsc{Elene} encodings of $u$ is shown in~\autoref{fig:EleneLDiagram}.

To embed edge $\langle u, w \rangle$ in $\mathcal{S}^{k}_v$, we use the quadruplets of $u$ and $w$, $q_u$ and $q_w$, and increase the granularity of distances to capture the \emph{relative} direction of the edge, following~\autoref{fig:3WLGraphsEleneEdgeDistinguishable}:
\begin{align*}
     \delta_{uw} = l_u - l_w + 1 \in \{0,1,2\}.
\end{align*}
We embed $\langle u, w \rangle$ in a permutation-invariant manner, summing embeddings bidirectionally so $\texttt{Emb}(u, w|v) = \texttt{Emb}(w, u|v)$:
\begin{align}\label{eq:EleneEdgeEmbeddings}
    \texttt{Emb}(u, w|v) = 
    \left(\textbf{W}_{(\textcolor{green!60!black}{l_u},\textcolor{green!60!black}{\delta_{uw}},\textcolor{red!60!black}{d^1_u})}^{1,\texttt{ed}} \bigconcat \textbf{W}_{(\textcolor{green!60!black}{l_u},\textcolor{green!60!black}{\delta_{uw}},\textcolor{blue!70!black}{d^2_u})}^{2,\texttt{ed}} \bigconcat \textbf{W}_{(\textcolor{green!60!black}{l_u},\textcolor{green!60!black}{\delta_{uw}},\textcolor{red!60!black}{d^3_u})}^{3,\texttt{ed}}\right) + 
    \Big(\textbf{W}_{(\textcolor{green!60!black}{l_w},\textcolor{green!60!black}{\delta_{wu}},\textcolor{red!60!black}{d^1_w})}^{1,\texttt{ed}} \bigconcat \textbf{W}_{(\textcolor{green!60!black}{l_w},\textcolor{green!60!black}{\delta_{wu}},\textcolor{blue!70!black}{d^2_w})}^{2,\texttt{ed}} \bigconcat \textbf{W}_{(\textcolor{green!60!black}{l_w},\textcolor{green!60!black}{\delta_{wu}},\textcolor{red!60!black}{d^3_w})}^{3,\texttt{ed}}\Big). 
\end{align}
$\textbf{W}^{1,\texttt{ed}}$, $\textbf{W}^{2,\texttt{ed}}$, and $\textbf{W}^{3,\texttt{ed}} \in \mathbb{R}^{3 \times S \times \omega}$ are edge-level embedding matrices with 3$\times$ more entries to represent the three possible values of $\delta_{uw}$. 
\begin{figure}[!ht]
    \centering
        \includegraphics[width=1.0\linewidth]{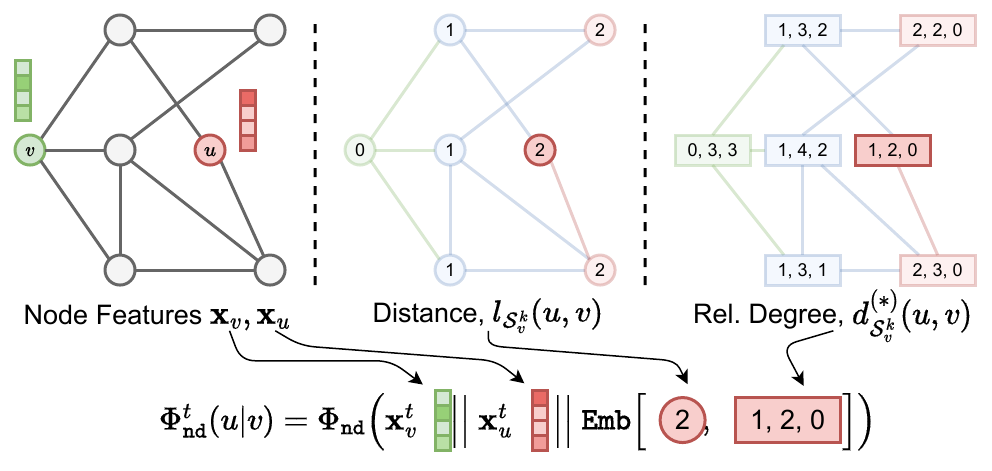}
        \caption{\textsc{Elene-L} encoding of \textcolor{red!80!black}{$u$} in the $k=2$ ego-network of \textcolor{green!50!black}{$v$}. The representation contains the feature vectors of both nodes (\textcolor{green!50!black}{$\mathbf{x}_v$} and \textcolor{red!80!black}{$\mathbf{x}_u$}, left), the distance information of $u$ to $v$ (\textcolor{red!80!black}{2, center}) and the relative degree information (\textcolor{red!80!black}{$[1, 2, 0]$, right}).}
    \label{fig:EleneLDiagram}
\end{figure}

\section{Related Work}
\label{sec:EleneRelWork} 

In this section, we connect related work with \textsc{Elene} encodings and the practical applications of \textsc{Elene} and \textsc{Elene-L} in~\autoref{sec:EleneLearning}. Per~\autoref{sec:EleneIntroduction}, the expressivity of MP-GNNs is often studied through the 1-WL test and its more powerful $k$-WL variants. Despite great successes in many domains~\citep{DuvenaudNIPS2015,battaglia_NIPS2016,gilmer2017neural,Ying-2018-PinSAGE}, the GIN architecture~\citep{xu2018how} showed that one-hop MP-GNNs are at most as expressive as 1-WL. This increased interest in expressive power within the community---in the formal study of MP-GNNs~\citep{pmlr-v162-papp22a}, and to boost message-passing  with spectral~\citep{balcilar2021breaking}, positional~\citep{pmlr-v97-you19b,DEA_GNNs,abboud2022spnn}, path-level~\citep{eliasof2022pathgcn,michel2023pathnns}, sub-graph~\citep{nikolentzos2020khop,zhang2021,bevilacqua2022equivariant,frasca2022understanding,mitton2023subgraph}, structural signals~\citep{morris2019,bodnar2021}, or their combination~\citep{ying2021do,zhao2022,dwivedi2022graph,rampasek2022GPS}. We now discuss the theoretical ability of models to express certain computations---\emph{expressivity} in the abstract---and empirical performance and architectures of graph learning methods.

\textbf{--- Expressivity.} The most common framework to study expressivity are the $k$-WL tests and its variants~\citep{morris2019}. Recent research has also focused on other perspectives, such as matrix languages~\citep{balcilar2021breaking}, or the GD-WL test~\citep{zhang2023rethinking}---which reframes expressivity in terms of graph biconnectivity, capturing the ability to identify cut nodes and edges. Shortest Path Neural Networks (SPNNs)~\citep{abboud2022spnn} introduced a model aggregating across shortest-path distances, but not edges or messages across neighbors, whose expressivity differs from 1-WL and addresses the over-squashing problem~\citep{alon2021on}. Finally, another approach has been through 2-variable counting logics~\citep{barceló2020logic,grohe2021logic}---studying what Boolean statements MP-GNNs can express.

\textsc{Elene} builds on previous expressivity analyses by presenting features that can distinguish challenging 3-WL equivalent graphs---\texttt{SRG}s. In~\autoref{sec:EleneExpressivity}, we will show that \textsc{Elene} can fully identify between \texttt{SRG}s with different parameters, and prove that an \textsc{Elene-L} model can emulate SPNNs. Furthermore, in~\autoref{sec:EleneExperiments} we empirically evaluate our models on the $h$-Proximity tasks and explore whether \emph{structural ($k$-WL) expressivity is all we need}. We find \textsc{Elene-L} outperforms previous strong baselines of SPNNs and Graphormers~\citep{ying2021do,abboud2022spnn} while simply concatenating \textsc{Elene} encodings underperforms, showing that \emph{expressivity} without \emph{attributes} is insufficient for certain tasks.

\textbf{--- Boosting Graph Neural Models.} Besides studying flavors of expressivity, researchers have also focused on improving performance for MP-GNNs and Graph Transformers. We summarize the most relevant families of novel network architectures in connection with \textsc{Elene} and \textsc{Elene-L}:

\textbf{--- Sub-graph MP-GNNs.} \textsc{Elene} is most related to equivariant, sub-graph methods---including $k$-hop GNNs~\citep{nikolentzos2020khop}, Structural MP-GNNs (\textbf{SMP})~\citep{vignac2020smp}, NestedGNNs (\textbf{NGNNs})~\citep{zhang2021}, Identity-GNNs (\textbf{ID-GNN})~\citep{you2021identity}, Equivariant Subgraph Aggregation Networks (\textbf{ESAN})~\citep{bevilacqua2022equivariant}, Ordered Subgraph Aggregation Networks (\textbf{OSAN})~\citep{qian2022ordered},  GNN-As-Kernel (\textbf{GNN-AK})~\citep{zhao2022}, Shortest Path Neural Networks (\textbf{SPNN})~\citep{abboud2022spnn}, Subgraph Union Networks (\textbf{SUN})~\citep{frasca2022understanding}, and Subgraph Permutation Equivariant Networks (\textbf{SPEN})~\citep{mitton2023subgraph}.
By encoding structural attributes of the ego-network sub-graph, \textsc{Elene} captures similar signals as GNN-AK's centroid encodings. However, \textsc{Elene-L} extends node and edge representations within sub-graphs \emph{first}, and then feeds sub-graph aware data into a GNN---rather than applying a GNN on the sub-graph and aggregating its outputs as in NGNNs and GNN-AK. During learning, \textsc{Elene-L} resembles SPEN and ESAN with the \texttt{EGO+} policy with node marking, as the root of the ego-network is implicitly marked by the relative degree and distance pairs. 

These sub-graph GNNs involve processing the sub-graphs during training and inference, which is avoided by approaches like \textsc{Igel}~\citep{alvarez-gonzalez2022beyond}, GSNs~\citep{bouritsas2023} or ESC-GNN~\citep{yan2023efficiently}, and also \textsc{Elene}---as they add sub-structure information without executing a GNN in the sub-graph. \autoref{th:EleneAsExpressive} shows \textsc{Elene} encodings are a superset of sparse \textsc{Igel} vectors. \textsc{Elene} requires no choice of substructure to count. In contrast, GSNs require counting $k$-node structures which has an exponential cost in $k$.
Finally, ESC-GNNs also use structural degree and distance signals directly as inputs. However, \textsc{Elene-L} learns additional embeddings from the structural encodings rather than using them as static features. 

Other approaches instead tackle the representation task by learning to select sub-graphs, such as MAG-GNN~\citep{kong2023maggnn} or Policy-Learn~\citep{bevilacqua2023efficient}---for which \textsc{Elene} signals could act as additional features.  Finally, \textsc{Elene-L} can be understood as a graph rewiring approach as recently exemplified by Dynamic Graph Rewiring (\textbf{DRew})~\citep{gutteridge2023drew}, since each $\textsc{Elene-L}$ layer can be independently parameterized to connect nodes via ego-networks \emph{and} edge-level sub-graphs---adding virtual edges between vertices that are $k$-hops away, and also passing signals across adjacent nodes (i.e. edges) whose $k$-depth ego-networks intersect. In this work, we only explore the impact of \emph{static} edge-level rewiring through relative degrees and Node or Edge-Centric sub-graphs.

The key difference with the aforementioned methods is that \textsc{Elene-L} captures edge-level information both in the encoding and during learning, as per~\autoref{fig:3WLGraphsEleneEdgeDistinguishable}. \autoref{subsec:EleneLExpressivity}, shows edge-level signals boost expressivity and corroborates results from \textbf{SUN} that node-only sub-graph models are upper-bounded by the 3-WL test~\citep{frasca2022understanding}. In~\autoref{subsec:EleneExpressivityResults}, we experimentally validate that \textsc{Elene-L} (\textbf{ED}) but not (\textbf{ND}) reaches 100\% accuracy on SR25, a challenging \texttt{SRG} dataset only solved before without graph perturbations by $\mathcal{O}(n^2)$ \textbf{PPGN-AK}~\citep{HaggaiNIPS2019PPGN,zhao2022}, and partially by \textbf{SPEN}~\citep{mitton2023subgraph}, which distinguished 97\% of non-isomorphic pairs.

 \textbf{--- Perturbation methods.} Beyond sub-graph methods, random perturbations of the graph structure like DropGNN~\citep{papp2021dropgnn}, Random Node Initializations~\citep{abboud2021}, or paths from random walks~\citep{eliasof2022pathgcn} have also shown surprising performance in expressivity tasks. Furthermore, random-walks based methods have been shown to be effective at capturing structural information, including positional information (\textbf{RWPE})~\citep{dwivedi2022graph}. Although \textsc{Elene} in its current definition does not consider graph perturbations or stochastic features, the underlying quadruplets can be easily adapted to ignore dropped-out nodes or edges, and can be seamlessly combined with random node initializations or global positional encodings.

 \textbf{--- Graph Transformers.} Similarly, the extension of Transformer models to graph tasks has led to increased research interest, notably with the introduction of Graphormer~\citep{ying2021do}---which included positional and degree encoding similar to \textsc{Elene}, but using only \emph{absolute} in/out degrees. More recently, Pure Graph Transformers~\citep{kim2022pure} removed graph-specific architecture choices, directly encoding nodes and edges as tokens processed through self-attention and a global read-out. 

Finally, a series of works have yielded high-performance recipes for graph transformers such as GPS~\citep{rampasek2022GPS}---combining strong inductive biases from MP-GNNs, as well as global and local encoding to build high-performance Graph Transformers. Transformers on graphs can be understood as \emph{fully-connected} graph processors, and it has been shown that Graphormers can be emulated through an SPNN~\citep{abboud2022spnn}. In~\autoref{subsec:EleneLSPNNNExpressivity}, we show that \textsc{Elene-L} can, in turn, emulate an SPNN---and transitively a Graphormer. We consider the analysis of edge-level \textsc{Elene} signals in Graph Transformers as future work, focusing our study on MP-GNN architectures.

\section{Expressive Power}\label{sec:EleneExpressivity}

We now analyze the expressive power of \textsc{Elene}---formally answering our question on \emph{which information is sufficient to distinguish \texttt{SRG}s}. We extend recent results on \textsc{Igel}, a sparse vector encoding similar to \textsc{Elene}~\citep{alvarez-gonzalez2022beyond} and show that Edge-Centric and Node-Centric \textsc{Elene} are strictly more expressive than previous methods relying on degrees and distances by comparing their expressivity on \texttt{SRG}s. We then show that \textsc{Elene-L} is at least as expressive as \textsc{Elene}, and prove that \textsc{Elene-L} (\textbf{ED}) is more expressive than \textsc{Elene-L} (\textbf{ND}) and \textsc{Elene} (\textbf{ND}). Finally, we connect our framework with SPNNs~\citep{abboud2022spnn}, showing that the latter can be expressed by an instance of node-centric \textsc{Elene-L} without edge-degree information---motivating our analysis on \emph{attributed} tasks in~\autoref{sec:EleneExperiments}. 

\subsection{Expressive Power of ELENE}

Previous work has shown that encoding-based and sub-graph MP-GNN methods are limited in their ability to distinguish 3-WL equivalent \texttt{SRG}s~\citep{arvind2020,balcilar2021breaking,alvarez-gonzalez2022beyond,frasca2022understanding}. Recently, \citet{alvarez-gonzalez2022beyond} presented \textsc{Igel}---a simple, sparse node feature vector containing counts of distance and degree tuples in an ego-network, showing it is strictly more expressive than the 1-WL test. Following~\autoref{eq:ELENESparseVector}, \textsc{Elene} multi-sets may also be represented as sparse vectors---which can then be used as feature vectors, but also to distinguish ego-network sub-graphs. 

We build on top of the results from~\citep{alvarez-gonzalez2022beyond} and show \textsc{Elene} is at least as expressive as \textsc{Igel}. We then find an upper-bound of expressivity for \textsc{Igel}, which is at most able to distinguish between $n$, $d$ or $\lambda$ parameters of $\texttt{SRG}$s, but not $\mu$, and show Node-Centric and Edge-Centric \textsc{Elene} is strictly more expressive than \textsc{Igel} on $\texttt{SRG}$s as it can explicitly encode all \texttt{SRG} parameters by counting edges:
\begin{theorem}\label{th:EleneAsExpressive}
    Node-Centric \textsc{Elene} is at least as expressive as \textsc{Igel}~\citep{alvarez-gonzalez2022beyond}, and transitively more expressive than 1-WL.
\end{theorem}
\begin{proof}
    The \textsc{Igel} encoding in~\citet{alvarez-gonzalez2022beyond} is a simpler version of~\autoref{eq:ELENESparseVector} that only considers distance ($\textcolor{green!60!black}{l_u}$) and \emph{absolute} degree ($\textcolor{blue!70!black}{d^2_u}$):
    \begin{align*}
        \textsc{Igel}_{\texttt{vec}}^{k}(v)_{i} = \bigg|\bigldblbrace (\textcolor{green!60!black}{l_u}, \textcolor{red!60!black}{d^1_u}, \textcolor{blue!70!black}{d^2_u}, \textcolor{red!60!black}{d^3_u}) \in e^{k}_v
             \Big| f'(\textcolor{green!60!black}{l_u}, \textcolor{blue!70!black}{d^2_u}) = i\bigrdblbrace\bigg|.
    \end{align*}
    where $f'(\textcolor{green!60!black}{l_u}, \textcolor{blue!70!black}{d^2_u})$ is a bijective function that does not consider \emph{relative} degrees, in contrast with \textsc{Elene}'s $f$. Thus, for any ego-network, \textsc{Elene} includes all information required to construct \textsc{Igel} vectors, so it is at least as expressive as \textsc{Igel}.
\end{proof}
\begin{theorem}\label{th:EleneMoreExpressive}
    \textsc{Elene} (\textbf{ND}) encodes and distinguishes $\texttt{SRG}$s with different parameters of $n$, $d$, $\lambda$ and $\mu$. 
\end{theorem}
\begin{proof}
    Consider 
    $\texttt{SRG}(n, d, \lambda, \mu) = (V, E)$. The maximum diameter of an \texttt{SRG} is 2~\citep{srg2022}, so we focus on the case where $k=2$. The \textsc{Elene} (\textbf{ND}) encoding of $v \in V$ according to~\autoref{eq:ELENEEncRaw} is: 
    \begin{align*}
        e_{v}^{2} = \bigldblbrace \Big(0, 0, d, d\Big)^1\mskip-10mu, \Big(1, 1, d, d\texttt{-}\lambda\texttt{-}1 \Big)^d\mskip-10mu,
        \Big(2, d-\mu, d, 0\Big)^{n\texttt{-}d\texttt{-}1}\bigrdblbrace
    \end{align*}
    By definition, any Node-Centric ego-network in an $\texttt{SRG}$ has a single root with $d$ neighbors, $d$ neighbors with one edge with the root and $d-\lambda-1$ edges to the next layer, and the remaining $n-d-1$ non-adjacent nodes to the root each have $d-\mu$ edges with the $d$ neighbors of the root. Consider $\texttt{SRG}'(n', d', \lambda', \mu') = (V', E')$. If any of the parameters between $\texttt{SRG}$ and $\texttt{SRG}'$ differ, so will $e_{v}^{2}$ from $e_{v'}^{2}$. This is not the case for \textsc{Igel}, which can at most capture $n$, $d$, and $\lambda$:
    \begin{align*}
        \textsc{Igel}_{v}^{1} &= \bigldblbrace \Big(0, d\Big)^1\mskip-10mu, \Big(1, 1+\lambda\Big)^d\bigrdblbrace\\
        \textsc{Igel}_{v}^{2} &= \bigldblbrace \Big(0, d\Big)^1\mskip-10mu, \Big(1, d\Big)^d\mskip-10mu,
        \Big(2, d\Big)^{n\texttt{-}d\texttt{-}1}\bigrdblbrace
    \end{align*}
    Thus, \textsc{Elene} (\textbf{ND}) can encode and distinguish all parameters of $\texttt{SRG}$s---outperforming \textsc{Igel}. However, \textsc{Elene} (\textbf{ND}) \emph{cannot distinguish} non-isomorphic $\texttt{SRG}$s when $n = n'$, $d = d'$, $\lambda = \lambda'$, and $\mu = \mu'$.
\end{proof}
\begin{corollary}
    \textsc{Elene} (\textbf{ND}) is more expressive than \textsc{Igel} and 1-WL, per~\autoref{th:EleneAsExpressive} \&~\autoref{th:EleneMoreExpressive}.
\end{corollary}
\begin{corollary}\label{co:EleneSRGLimit}
    \textsc{Elene} (\textbf{ND}) signals at the node-level are not capable of distinguishing between non-isomorphic $\texttt{SRG}$s with equal parameters---e.g. the graphs in~\autoref{fig:3WLGraphsEleneEdgeDistinguishable}.
\end{corollary}
\begin{proposition}
    \textsc{Elene} (\textbf{ED}, leveraging both $e_v^k\ \forall v \in V$ and $e^{k}_{\langle u, v \rangle}\ \forall (u, v) \in E$) is strictly more expressive than \textsc{Elene} (\textbf{ND}), as it can distinguish the pair of graphs in~\autoref{fig:3WLGraphsEleneEdgeDistinguishable}.
\end{proposition}

\subsection{Expressive Power of ELENE-L}\label{subsec:EleneLExpressivity}

\begin{theorem}
    \textsc{Elene-L} with the sum as the pooling operator is at least as expressive as \textsc{Elene}.
\end{theorem}
\begin{proof}
    We first show \textsc{Elene-L} (\textbf{ND}) is at least as expressive as \textsc{Elene} (\textbf{ND}). We then show that the \textbf{ED} variants are at least as expressive as \textbf{ND} variants (\autoref{prop:EleneLEdge}), and show through~\autoref{fig:3WLGraphsEleneEdgeDistinguishable} that \textsc{Elene-L} (\textbf{ED}) is more powerful than Node-Centric \textsc{Elene-L} (\textbf{ND}).
    
    \textbf{On} \textsc{Elene-L} (\textbf{ND}). $\forall v \in V$, the $\textsc{Elene-L}(\textbf{x}_v^{t})$ representation of $v$ is given by $\Phi_{\texttt{ND}}^{t}(v)$ as per~\autoref{eq:ELENE-L-NodeMessage}. $\Phi_{\texttt{ND}}^{t}(v)$ is the result of applying $\Phi_{\texttt{ND}}$ to the concatenation of $\textbf{x}_v^t$ and the combined representations of every $u \in \mathcal{V}^{k}_v$ and $\langle u, w\rangle \in \mathcal{E}^{k}_v$. Let $\Phi_{\texttt{out}}$ and $\Phi_{\texttt{nd}}$ be the identity function, we exclude edge-level information by discarding the output of $\Phi_{\texttt{ed}}$. We now expand $\hat{\Phi}_{\texttt{ND}}^{t}(v)$, which is $\Phi_{\texttt{ND}}^{t}(v)$ with the changes to the learnable $\Phi$:
    \begin{align*}
        \hat{\Phi}_{\texttt{ND}}^{t}(v) = \bigg(\textbf{x}_v^{t} \biggconcat \sum_{u}^{\mathcal{V}^{k}_v} \Big(\textbf{x}_v^{t} \bigconcat \textbf{x}_u^{t} \bigconcat \texttt{Emb}(u|v)\Big)\bigg).
    \end{align*}
    We discard repeated $\textbf{x}_v^{t}$ terms, and rewrite the representation of $v$, distributing the sum over the concatenated vector:
    \begin{align*}
        \hat{\Phi}_{\texttt{ND}}^{t}(v) = \bigg(\textbf{x}_v^{t} \biggconcat \sum_{u}^{\mathcal{V}^{k}_v} \textbf{x}_u^{t} \biggconcat \sum_{u}^{\mathcal{V}^{k}_v} \texttt{Emb}(u|v)\bigg).
    \end{align*}
    Let $\textbf{W}^{1,\texttt{nd}}, \textbf{W}^{2,\texttt{nd}}, \textbf{W}^{3,\texttt{nd}} \in \mathbb{R}^{S \times S}$ used by $\texttt{Emb}$ be identity matrices so every relative degree and distance pair out of $S = d_{\texttt{max}} \cdot (k + 1)$ has a single position in $\textbf{W}$. By using the sum as the pooling function, we obtain the frequency of each relative degree and distance pair, matching \textsc{Elene} in~\autoref{eq:ELENESparseVector}. Thus, $\hat{\Phi}_{\texttt{out}}^{t}(v)$ contains the information contained in the \textsc{Elene} multi-set, reaching at least the same expressivity. \qed
    \begin{proposition}\label{prop:EleneLEdge}
        \textsc{Elene-L} (\textbf{ED}) variants with the sum as the pooling operator are at least as expressive as \textsc{Elene}.
    \end{proposition}
    \textbf{On} \textsc{Elene-L} (\textbf{ED}). We had discarded $\Phi_{\texttt{ed}}$, showing \textbf{ED} variants are at least as expressive as \textbf{ND} variants, since the concatenation of edge-level information can only match or boost expressivity. Thus, (\textbf{ND}) and (\textbf{ED}) variants of \textsc{Elene-L} are as expressive as \textsc{Elene}.
\end{proof}
\begin{theorem}\label{th:EleneLMoreExpressive}
    \textsc{Elene-L} (\textbf{ED}) is more expressive than \textsc{Elene-L} (\textbf{ND}) and \textsc{Elene} (\textbf{ND}).
\end{theorem}
\begin{proof}
    There is at least a pair of non-isomorphic $\texttt{SRG}$s that \textsc{Elene-L} (\textbf{ED}) can distinguish. In~\autoref{subsec:EleneIntuition}, we show that the Shrikhande and $4\times4$ Rook graphs~\citep{arvind2020,balcilar2021breaking} (parametrized as $\texttt{SRG}(16, 6, 2, 2)$) can be distinguished by Edge-Centric counts that \textsc{Elene-L} (\textbf{ED}) captures despite being undistinguishable by 3-WL (by implementing~\autoref{eq:ELENEEncRaw} at the \emph{edge} level). Following from~\autoref{th:EleneAsExpressive} and \autoref{th:EleneMoreExpressive}, both graphs are indistinguishable by $\textsc{Elene}$ (\textbf{ND}) or $\textsc{Elene-L}$ (\textbf{ND}), as well as sub-graph GNNs like GNN-AK or SUN.  
 
    Intuitively, \texttt{SRG}s are indistinguishable with node-centric $k$ ego-network sub-graph encodings when $k \in \{1, 2\}$ since all nodes produce identical representations, as shown in~\autoref{th:EleneMoreExpressive}. However, the graphs can be distinguished by edge-level information as per~\autoref{eq:ELENE-L-Edge}, as the intersection of $k$-depth ego-networks for $\langle v_1, v_2 \rangle$ differ in edge counts between both \texttt{SRG}s---as observed in~\autoref{sec:EleneDef}. 
    
    We can see the $4\times4$ Rook Graph has 6 edges (i.e. $|\mathcal{E}^k_{\langle v_1, v_2 \rangle})| = 6$) while the Shrikhande graph has $|\mathcal{E}^k_{\langle v_1, v_2 \rangle})| = 5$, hence the graphs are distinguishable by \textsc{Elene-L} (\textbf{ED}), but not \textsc{Elene-L} (\textbf{ND}) or \textsc{Elene} (\textbf{ND}).
\end{proof}

\subsection{Linking ELENE and Shortest Path Neural Networks.}\label{subsec:EleneLSPNNNExpressivity}

\begin{remark}\label{th:SPNNsGraphormers}
    A Graphormer with max. shortest path length $M$ and global readout is an instance Shortest Path Neural Networks (SPNNs) with $k = M - 1$ depth~\citep{abboud2022spnn}.
\end{remark}
\begin{theorem}\label{th:EleneLSPNNExpressive}
    \textsc{Elene-L} (\textbf{ND}) is as expressive as Shortest Path Neural Networks (SPNNs), and transitively, Graphormers.
\end{theorem}
\begin{proof}
    Let $\mathcal{L}_G^{k}(v) = \{u|u \in V \wedge l_G(u, v) = k\}$ be the nodes in $G$ exactly at distance $k$ of $v$. In~\cite{abboud2022spnn}, a $k$-depth SPNN updates the hidden state of node $v$ an aggregation over the $1, ..., k$ exact-distance neighbourhoods: 
    \begin{align}\label{eq:SPNNDefinition}
        \textbf{x}_v^{t + 1} = \Phi_{\texttt{sp}}\Big((1 + \epsilon) \cdot \textbf{x}_v^{t} + \sum_{i=1}^{k} \alpha_i \sum_{u \in \mathcal{L}_G^{i}(v)} \textbf{x}_u^{t}\Big).
    \end{align}
    We show that \textsc{Elene-L} (\textbf{ND}) can implement SPNNs. First, let $\gamma_{\texttt{nd}} = 1$ in~\autoref{eq:ELENE-L}, such that: 
    \begin{align*}
        \textbf{x}_v^{t + 1} &= \Phi_{\texttt{ND}}^{t}(v) = 
        \Phi_{\texttt{ND}}\Bigg(\textbf{x}_v^{t} \biggconcat \sum_{u}^{\mathcal{V}^{k}_v} \Phi_{\texttt{nd}}^{t}(u|v) \biggconcat  \sum_{\langle u, w\rangle}^{\mathcal{E}^{k}_v} \Phi_{\texttt{ed}}^{t}(u, w|v)\Bigg).
    \end{align*}
    We drop $\Phi_{\texttt{ed}}^{t}(u, w|v)$ as SPNNs ignore edge-level signals\footnote{Including edge-level signals may bring \textsc{Elene-L} (\textbf{ED}) to parity with Pure Graph Transformers~\citep{kim2022pure}. We do not explore this connection.}. Let $\Phi_{\texttt{ND}}(\cdot)$ be composed of two functions $\Phi_{\texttt{sp}}(g(\cdot))$\footnote{$g(\cdot)$ is a linear combination over concatenated input vectors, learnable by a first layer of $\Phi_{\texttt{out}}$ without activations.} where:
    \begin{align*}
        g\Big(\textbf{x}_v^{t} \bigconcat \sum_{u}^{\mathcal{V}^{k}_v} \Phi_{\texttt{nd}}^{t}(u|v)\Big) = \Big((1 + \epsilon) \cdot \textbf{x}_v^{t} + \sum_{u}^{\mathcal{V}^{k}_v} \Phi_{\texttt{nd}}^{t}(u|v)\Big).
    \end{align*}
    We replace $\Phi_{\texttt{ND}}$ by $\Phi_{\texttt{sp}}$ and $g(\cdot)$, and expand $\Phi_{\texttt{nd}}^{t}(u|v)$: 
    \begin{align*}
        \textbf{x}_v^{t + 1} = \Phi_{\texttt{sp}}\Big((1 + \epsilon) \cdot \textbf{x}_v^{t} + \sum_{u}^{\mathcal{V}^{k}_v} \Phi_{\texttt{nd}}\Big(\textbf{x}_v^{t} \bigconcat \textbf{x}_u^{t} \bigconcat \texttt{Emb}(u|v)\Big).
    \end{align*}
    We then instantiate $\Phi_{\texttt{nd}}(\cdot)$ as:
    \begin{align*}
        \Phi_{\texttt{nd}}(\cdot) = \sum_i^k \alpha_i \cdot \texttt{if}[i = l_{\mathcal{S}}(u, v)] \cdot \textbf{x}_u^{t}
    \end{align*}
    \texttt{if}$[\cdot]$ can be implemented through the distance and degree signals in $\texttt{Emb}$, such we can check if the node distance matches a specific value\footnote{This is not necessary during learning: a one-hot `decoder' can be implemented using a two-layer perceptron with ReLU activations.}. Substituting in $\textbf{x}_v^{t + 1}$ above yields:
    \begin{align*}
        \textbf{x}_v^{t + 1} = \Phi_{\texttt{sp}}\Big((1 + \epsilon) \cdot \textbf{x}_v^{t} + \sum_{u}^{\mathcal{V}^{k}_v} \sum_i^k \alpha_i \cdot \texttt{if}[i = l_{\mathcal{S}}(u, v)] \cdot \textbf{x}_u^{t}\Big),
    \end{align*}
    which is equivalent to~\autoref{eq:SPNNDefinition}, and shows \textsc{Elene-L} (\textbf{ND}) can learn like SPNNs---and, transitively through~\autoref{th:SPNNsGraphormers}, that Graphormers can be emulated by \textsc{Elene-L} (\textbf{ND}).
\end{proof}

\section{Experimental Results}\label{sec:EleneExperiments}
We now study the effect of introducing \textsc{Elene} and \mbox{\textsc{Elene-L}} in a variety of graph-level settings, evaluating where \emph{purely structural} \textsc{Elene} encodings underperform \textsc{Elene-L}, and the practical impact of \textsc{Elene} variants in terms of model performance, training time, and memory costs. We describe our experimental protocol in~\autoref{subsec:EleneExpProtocol} and provide reproducible code, hyper-parameters, and analysis scripts through Github\footnote{\url{https://github.com/nur-ag/ELENE}} for four experimental benchmarks:

\textbf{A) Expressivity}. Evaluates whether models distinguish non-isomorphic graphs (on 1-WL EXP~\citep{abboud2021} and 3-WL SR25~\citep{balcilar2021breaking} equiv. datasets), count sub-graphs (in RandomGraph~\citep{chen2022CanGNNsCountSubstructures}), and evaluate graph-level properties~\citep{corso2020PNAGraphProp}.
    
\textbf{B) Proximity}. Measures whether models learn long-distance \emph{attributed} node relationships in $h$-Proximity datasets~\citep{abboud2022spnn}.
    
\textbf{C) Real World Graphs}. Evaluates performance on five large-scale graph classification/regression datasets from Benchmarking GNNs (ZINC, CIFAR10, PATTERN)~\citep{dwivedi2020benchmarkgnns}, and the Open Graph Benchmark (MolHIV, MolPCBA)~\citep{hu2020ogb}.

\textbf{D) Memory Scalability}. Evaluates the memory consumption of \textsc{Elene-L} on $d$-regular graphs, varying $n$ and $d_{\texttt{max}}$ to validate the algorithmic complexity analysis in~\autoref{subsec:ELENEAlgorithmicAnalysis} and comparing with the memory consumption of GIN-AK, GIN-AK\textsuperscript{+} and SPEN~\citep{mitton2023subgraph}.

\subsection{Experimental Protocol}\label{subsec:EleneExpProtocol}

\textbf{Reporting.} When reported in the original studies, we show stddevs for experiments with more than two runs following~\citep{zhao2022}, and highlight best-performing models per task in \textbf{\underline{underlined bold}}. \textsc{Elene} denotes~\autoref{eq:ELENESparseVector} as additional features, while \textsc{Elene-L} denotes the representations of~\autoref{eq:ELENE-L} and~\autoref{eq:ELENE-L-Edge}. \textbf{(ED)} denotes \textsc{Elene-L} with Edge-Centric signals, while \textbf{(ND)} denotes a Node-Centric variant that ignores edge information for ablation studies. `\textsuperscript{$\dagger$}' indicates results from the literature.

\textbf{Environment.} Experiments ran on a shared server with a 48GB Quadro RTX 8000 GPU, 40 CPU cores and 502GB RAM. Each individual job has a limit of 96GB RAM and 8 CPU cores. To measure memory and time costs without sharing resources, we also reproduced our experiments on real-world graphs on a SLURM cluster with nodes equipped with 22GB Quadro GPUs. Finally, scalability experiments ran on Tesla T4 GPUs with 15.11GB of VRAM to validate our approach on consumer hardware.

\textbf{Experimental Setup.} We explore sub-sets of \textsc{Elene} hyper-parameters via grid search with $k \in \{0, 1, 2, 3, 5\}$ parameter ranges for \textsc{Elene} and \textsc{Elene-L}, and test the ED/ND variants for \textsc{Elene-L} with embedding params. $\omega \in \{16, 32, 64\}$, $\rho = d_{\texttt{max}}$, using masked-mean pooling for stability. For $h$-Proximity~\citep{abboud2022spnn}, we compare against SPNNs~\citep{abboud2021} and Graphormer\citep{ying2021do} as originally reported. For Expressivity and Real World Graphs, we reuse hyper-parameters and splits from GIN-AK\textsuperscript{+} in~\citet{zhao2022} without architecture search, comparing against strong MP-GNN baselines from literature where GNN-AK\textsuperscript{+} underperforms: CIN~\citep{bodnar2021} for ZINC and SUN~\citep{frasca2022understanding} for sub-graph counting. We choose GINE~\citep{hu2020pretraining}, an edge-aware variant of GIN~\citep{xu2018how}, as our base MP-GNN given that GIN-AK\textsuperscript{+} outperforms its uplifted counterparts for GCN-AK\textsuperscript{+} and PNA-AK\textsuperscript{+}~\citep{DBLP:journals/corr/KipfW16-SemiSup-GCN,corso2020PNAGraphProp,zhao2022}, without running into out-of-memory issues like PPGN~\citep{HaggaiNIPS2019PPGN} in the PPGN-AK instantiation. Finally, for scalability we compare with GNN-AK on benchmark datasets~\cite{zhao2022} and SPEN~\cite{mitton2023subgraph}.

\textbf{Experimental Objectives.} We connect \emph{expressivity} and its relation to graph \emph{attributes}, comparing against methods that \emph{do not} perturb graph structure, e.g. DropGNN~\citep{papp2021dropgnn}; leverage random walks, e.g. RWPE~\citep{dwivedi2022graph}; or require costly pre-processing e.g. $\mathcal{O}(n^3)$ spectral eigendecompositions, such as GNNML3, LWPE, or GraphGPS~\citep{balcilar2021breaking,dwivedi2022graph,rampasek2022GPS}. Per~\autoref{subsec:EleneLSPNNNExpressivity}, \textsc{Elene} relates to Graphormers via SPNNs, so we focus on sub-graph GNNs and SPNNs.

\subsection{Expressivity}\label{subsec:EleneExpressivityResults}

We test \textsc{Elene} on four MP-GNN expressivity datasets, with results captured in~\autoref{tab:ExpressivityResults}. Introducing \textsc{Elene} signals improves the performance of GINs, and our single-run results EXP and SR25 are consistent with our formal analysis on~\autoref{sec:EleneExpressivity}---namely, \textsc{Elene} and \textsc{Elene-L} (\textbf{ND}) and (\textbf{ED}) all reach 100\% accuracy on the 1-WL equivalent EXP task, as expected from~\autoref{th:EleneAsExpressive}. Furthermore, \textsc{Elene-L} (\textbf{ED}) can distinguish all 3-WL equivalent \texttt{SRG}s in the challenging SR25 dataset---providing empirical evidence for~\autoref{th:EleneLMoreExpressive}.

\begin{table}[!ht]
\caption[Expressivity benchmark results.]{Expressivity benchmark results. In EXP and SR25, introducing \textsc{Elene-L} yields the best performance per task, shown in \textbf{\underline{underlined bold}}. We highlight the \emph{best-performing configurations from \textsc{Elene} variants on GIN in italics}, which we consistently observe in the Node-Centric (\textbf{ND}) configuration except for isomorphism tasks.}
\label{tab:ExpressivityResults}
\smaller
\centering
\begin{tabular}{@{}rrrrrrrrrr@{}}
\toprule
\multicolumn{1}{c}{\multirow{2}{*}{\textbf{Model}}}                                                                           & \multicolumn{1}{c}{\multirow{2}{*}{\textbf{\begin{tabular}[c]{@{}c@{}}EXP\\ (Acc.)\end{tabular}}}} & \multicolumn{1}{c}{\multirow{2}{*}{\textbf{\begin{tabular}[c]{@{}c@{}}SR25\\ (Acc.)\end{tabular}}}} & \multicolumn{4}{c}{\textbf{\begin{tabular}[c]{@{}c@{}}Count. Substr.\\ (MAE)\end{tabular}}}                                                             & \multicolumn{3}{c}{\textbf{\begin{tabular}[c]{@{}c@{}}Graph Prop. \\ ($\log_{10}$(MAE))\end{tabular}}}          \\
\multicolumn{1}{c}{}                                                                                                          & \multicolumn{1}{c}{}                                                                               & \multicolumn{1}{c}{}                                                                                & \multicolumn{1}{c}{\textbf{Tri.}} & \multicolumn{1}{c}{\textbf{Tail Tri.}} & \multicolumn{1}{c}{\textbf{Star}} & \multicolumn{1}{c}{\textbf{4-Cycle}} & \multicolumn{1}{c}{\textbf{IsCon.}} & \multicolumn{1}{c}{\textbf{Diam.}} & \multicolumn{1}{c}{\textbf{Radius}} \\ \midrule
\textbf{GIN}                                                                                                                  & 50\%                                                                                               & 6.67\%                                                                                              & 0.357                             & 0.253                                    & 0.023                             & 0.231                                & -1.914                               & -3.356                             & -4.823                              \\ \midrule
\textbf{SUN\textsuperscript{$\dagger$,\citep{frasca2022understanding}}}                                                        & \multicolumn{1}{c}{---}                                                                            & \multicolumn{1}{c}{---}                                                                             & {\ul \textbf{0.008}}              & {\ul \textbf{0.008}}                     & {\ul \textbf{0.006}}              & {\ul \textbf{0.011}}                 & -2.065                               & -3.674                             & {\ul \textbf{-5.636}}               \\
\textbf{GIN-AK\textsuperscript{$\dagger$,\citep{zhao2022}}}                                                                    & {\ul \textbf{100\%}}                                                                               & 6.67\%                                                                                              & 0.093                             & 0.075                                   & 0.017                             & 0.073                                & -1.993                               & -3.757                             & -5.010                              \\
\textbf{GIN-AK\textsuperscript{+}}                                                                                            & {\ul \textbf{100\%}}                                                                               & 6.67\%                                                                                              & 0.011                             & 0.010                                    & 0.016                             & {\ul \textbf{0.011}}                 & -2.512                               & -3.917                             & -5.260                              \\ \midrule
\textbf{GIN+ELENE}                                                                                                   & {\ul \textbf{100\%}}                                                                               & 6.67\%                                                                                              & 0.024                             & 0.023                                    & 0.020                             & 0.041                                & -2.218                               & -3.656                             & -5.024                              \\
\textbf{GIN+ELENE-L (\textbf{ND})}                                                                                   & {\ul \textbf{100\%}}                                                                               & 6.67\%                                                                                              & \textit{0.012}                    & \textit{0.015}                           & \textit{0.014}                    & \textit{0.016}                       & \textit{-2.620}                      & \textit{-3.815}                    & \textit{-5.117}                     \\
\textbf{GIN+ELENE-L (\textbf{ED})}                                                                                    & {\ul \textbf{100\%}}                                                                               & {\ul \textbf{100\%}}                                                                                & 0.023                             & 0.023                                    & 0.017                             & 0.023                                & -2.497                               & -3.541                             & -4.755                              \\ \midrule
\textbf{\begin{tabular}[c]{@{}r@{}}Best (GIN / GIN-AK) + \\ (ELENE / ELENE-L)\end{tabular}} & {\ul \textbf{100\%}}                                                                               & {\ul \textbf{100\%}}                                                                                               & 0.010                             & 0.010                                    & 0.014                             & {\ul \textbf{0.011}}                 & {\ul \textbf{-2.715}}                & {\ul \textbf{-4.072}}              & -5.267                              \\ \bottomrule
\end{tabular}
\end{table}

On Graph Properties and Counting Substructures (2 runs averaged, as in~\citet{zhao2022}), a GIN + \textsc{Elene-L} (\textbf{ND}) model consistently outperforms GIN-AK \emph{without context encoding}. In Counting, both \textsc{Elene} variants and GIN-AK\textsuperscript{+} are outperformed by SUN, but GIN+\textsc{Elene} matches or outperforms GIN on every task, showing that \textsc{Elene} features are informative and can boost performance by themselves. 

On both tasks, we find that GIN+\textsc{Elene-L} (\textbf{ED}) performs poorly---outperforming GIN+\textsc{Elene} but not our baselines. This might be caused by model over-parametrization, as six node and edge-level embedding matrices are learned for $3$ and $6$ layers on Counting Substructures and Graph Properties respectively\footnote{Weight sharing may help over-parametrization by learning a single structural representation, trading off expressivity.}. Finally, on the Graph Properties tasks of \texttt{IsConnected} and \texttt{Diameter}, a GIN-AK\textsuperscript{+} with \textsc{Elene-L} outperforms state-of-the-art results---and interestingly a GIN with \textsc{Elene-L} (\textbf{ND}) outperforms all existing baselines on the \texttt{IsConnected} task. This can be further improved by using a GIN-AK\textsuperscript{+} with \textsc{Elene-L} (ND). 

\subsection{\texorpdfstring{$h$}{h}-Proximity}\label{subsec:EleneProximityResults}

We evaluate \textsc{Elene-L} on $h$-Proximity~\citep{abboud2022spnn} tasks (10-fold averaged)---where nodes are assigned colors including red and blue, and models classify whether all red nodes have at most two blue nodes within $h$ hops (positive) or otherwise (negative), as in~\autoref{fig:hProximity}. Models must learn which colors are relevant for the target and capture long-ranging dependencies during learning. Edge information is irrelevant, and pre-computed encodings like \textsc{Elene} cannot capture interactions of distances and node attributes.

In~\citet{abboud2022spnn}, the authors reported that MP-GNNs perform well on $h=1$-Proximity, so we focus on the $h \in \{3, 5, 8, 10\}$ variants. \autoref{tab:hProximityResults} shows our results, where \textsc{Elene-L} (\textbf{ND}) outperforms strong baselines from SPNNs and Graphormer~\citep{abboud2021}. As expected, a GIN + \textsc{Elene} did not meaningfully improve over GIN. Our numerical results provide empirical validation for~\autoref{th:EleneLSPNNExpressive}.

\begin{table}[!ht]
\caption[$h$-Proximity binary classification results.]{$h$-Proximity binary classification results (accuracy). \textsc{Elene-L} \textbf{(ND)} without degree information outperforms baselines strong SPNNs and Graphormer baselines from\textsuperscript{$\dagger$}~\citet{abboud2021}.}\label{tab:hProximityResults}
\small
\centering
{%
\begin{tabular}{@{}r@{~}rrrr@{}}
\toprule
                              & \multicolumn{1}{c}{\textbf{3-Prox.}} & \multicolumn{1}{c}{\textbf{5-Prox.}} & \multicolumn{1}{c}{\textbf{8-Prox.}} & \multicolumn{1}{c}{\textbf{10-Prox.}} \\ \midrule
\textbf{GCN}\textsuperscript{$\dagger$}                  & $50.0 \pm 0.0$                           & $50.0 \pm 0.0$                           & $50.1 \pm 0.0$                           & $49.9 \pm 0.0$                            \\
\textbf{GAT}\textsuperscript{$\dagger$}                  & $50.4 \pm 1.0$                           & $49.9 \pm 0.0$                           & $50.0 \pm 0.0$                           & $50.0 \pm 0.0$                            \\ \midrule
\textbf{SPNN ($k=1$)}\textsuperscript{$\dagger$}         & $50.5 \pm 0.7$                           & $50.2 \pm 1.0$                           & $50.0 \pm 0.9$                           & $49.8 \pm 0.8$                            \\
\textbf{SPNN ($k=5$)}\textsuperscript{$\dagger$}    & $95.5 \pm 1.6$                           & $96.8 \pm 0.7$                           & $96.8 \pm 0.6$                           & $96.8 \pm 0.6$                            \\
\textbf{Graphormer}\textsuperscript{$\dagger$}           & $94.7 \pm 2.7$                           & $95.1 \pm 1.8$                           & $97.3 \pm 1.4$                           & $96.8 \pm 2.1$                            \\ \midrule
\textbf{GIN+ELENE} & $52.0 \pm 2.0$ & $51.8 \pm 1.2$ & $52.4 \pm 2.6$ & $51.4 \pm 1.1$ \\
\textbf{GIN+ELENE-L (\textbf{ND})} & \underline{$\mathbf{98.3 \pm 0.5}$}       & \underline{$\mathbf{98.6 \pm 0.5}$}       & \underline{$\mathbf{99.0 \pm 0.5}$}       & \underline{$\mathbf{99.2 \pm 0.3}$}        \\ \bottomrule
\end{tabular}%
}
\end{table}

\subsection{Real World Graphs}\label{subsec:EleneLargeBenchmark}

We also evaluate \textsc{Elene} and \textsc{Elene-L} on five real-world, large-scale graph classification and regression tasks. We test \textsc{Elene} and \textsc{Elene-L} on ZINC, MolHIV, PATTERN, CIFAR10, and MolPCBA and report our results in~\autoref{tab:RealWorldBenchmarkResults}. Given increased memory and computation costs and the weaker performance of \textsc{Elene-L} (\textbf{ED}) in~\autoref{subsec:EleneExpressivityResults}, we only evaluate \textsc{Elene-L} (\textbf{ND}). 

On ZINC, GIN + \textsc{Elene-L} (3 averaged runs) achieves comparable results to existing baselines, including SUN~\citep{frasca2022understanding}. Furthermore, by introducing \textsc{Elene-L} on GIN-AK\textsuperscript{+}, the model matches the previous strong baseline achieved by CIN~\citep{bodnar2021}. On PATTERN (3 averaged runs), GIN + \textsc{Elene-L} achieves comparable results to GIN-AK\textsuperscript{+}, but does not meet the best reported performance of GCN-AK\textsuperscript{+} by a 0.07\% delta. We do not achieve to independently reproduce GIN-AK\textsuperscript{+} results~\citep{zhao2022} on MolHIV (5 averaged runs)---finding that GIN with \textsc{Elene} or \textsc{Elene-L} do not have statistically significant ($p < 0.01$) differences with GIN, while the performance of GIN-AK\textsuperscript{+} is statistically inferior.

\begin{table}[!t]
\caption[Results on real world benchmark datasets.]{Results on real world benchmark datasets. We compare with published results and reproduce the experiments of~\cite{zhao2022}. Adding \textsc{Elene} variants to GIN and GIN-AK\textsuperscript{+} yield state-of-the-art results on ZINC and MolPCBA, and match the performance of existing methods in PATTERN and MolHIV.}
\label{tab:RealWorldBenchmarkResults}
\centering
\small
\begin{tabular}{@{}rrrrrr@{}}
\toprule
\multicolumn{1}{l}{}                                                                                                                              & \multicolumn{1}{c}{\textbf{\begin{tabular}[c]{@{}c@{}}ZINC\\ {\scriptsize(MAE)}\end{tabular}}}                        & \multicolumn{1}{c}{\textbf{\begin{tabular}[c]{@{}c@{}}PATTERN\\ {\scriptsize(Acc.)}\end{tabular}}}                      & \multicolumn{1}{c}{\textbf{\begin{tabular}[c]{@{}c@{}}MolHIV \\ {\scriptsize(ROC)}\end{tabular}}}                   & \multicolumn{1}{c}{\textbf{\begin{tabular}[c]{@{}c@{}}CIFAR10 \\ {\scriptsize(Acc.)}\end{tabular}}}                   & \multicolumn{1}{c}{\textbf{\begin{tabular}[c]{@{}c@{}}MolPCBA \\ {\scriptsize(AP)}\end{tabular}}}                   \\ \midrule
\textbf{GSN\textsuperscript{$\dagger$}}                                                                                      & $0.115 \pm 0.012$                                                                                               & \multicolumn{1}{c}{---}                                                                                           & $77.99 \pm 1.00$                                                                                              & \multicolumn{1}{c}{---} & \multicolumn{1}{c}{---} \\
\textbf{NGNN\textsuperscript{$\dagger$}}                                                                                         & \multicolumn{1}{c}{---}                                                                                         & \multicolumn{1}{c}{---}                                                                                           & $78.34 \pm 1.86$                                                                                              & \multicolumn{1}{c}{---} & $28.32 \pm 0.41$ \\
\textbf{CIN\textsuperscript{$\dagger$}}                                                                                         & $0.079 \pm 0.006$                                                                                               & \multicolumn{1}{c}{---}                                                                                           & {\ul \textbf{$\mathbf{80.94 \pm 0.57}$}}                                                                      & \multicolumn{1}{c}{---} & \multicolumn{1}{c}{---} \\
\textbf{SUN\textsuperscript{$\dagger$}}                                                                            & $0.083 \pm 0.003$                                                                                               & \multicolumn{1}{c}{---}                                                                                           & $80.55 \pm 0.55$                                                                                              & \multicolumn{1}{c}{---} & \multicolumn{1}{c}{---} \\ \midrule
\textbf{GCN-AK\textsuperscript{+}\textsuperscript{$\dagger$}}                                                                     & $0.127 \pm 0.004$                                                                                               & {\ul \textbf{$\mathbf{86.887 \pm 0.009}$}}                                                                        & $79.28 \pm 1.01$ & \underline{$\mathbf{72.70 \pm 0.29}$} & $0.285 \pm 0.000$                                                                                           \\
\textbf{GIN-AK\textsuperscript{$\dagger$}} & $0.094 \pm 0.005$ & $86.803 \pm 0.044$ & $78.29 \pm 1.21$ & $67.51 \pm 0.21$ & $0.274 \pm 0.000$ \\
\textbf{\begin{tabular}[c]{@{}r@{}}GIN-AK\textsuperscript{+}\\ {\scriptsize(Lit. results\textsuperscript{$\dagger$})}\end{tabular}} & \begin{tabular}[c]{@{}r@{}}$0.082 \pm 0.003$\\ {\scriptsize\textsuperscript{$\dagger$}$0.080 \pm 0.001$}\end{tabular} & \begin{tabular}[c]{@{}r@{}}$86.868 \pm 0.028$\\ {\scriptsize\textsuperscript{$\dagger$}$86.850 \pm 0.057$}\end{tabular} & \begin{tabular}[c]{@{}r@{}}$77.37 \pm 0.31$\\ {\scriptsize\textsuperscript{$\dagger$}$79.61 \pm 1.19$}\end{tabular} & \begin{tabular}[c]{@{}r@{}}$72.39 \pm 0.38$ \\ {\scriptsize\textsuperscript{$\dagger$}$72.19 \pm 0.13$} \end{tabular} & \begin{tabular}[c]{@{}r@{}}$0.293 \pm 0.004$ \\ {\scriptsize\textsuperscript{$\dagger$}$0.293 \pm 0.004$} \end{tabular} \\ \midrule
\textbf{\begin{tabular}[c]{@{}r@{}}GIN\\ {\scriptsize(Lit. results\textsuperscript{$\dagger$})}\end{tabular}}                       & \begin{tabular}[c]{@{}r@{}}$0.155 \pm 0.005$\\ {\scriptsize\textsuperscript{$\dagger$}$0.163 \pm 0.004$}\end{tabular} & \begin{tabular}[c]{@{}r@{}}$85.692 \pm 0.042$\\ {\scriptsize\textsuperscript{$\dagger$}$85.732 \pm 0.023$}\end{tabular} & \begin{tabular}[c]{@{}r@{}}$78.72 \pm 0.54$\\ {\scriptsize\textsuperscript{$\dagger$}$78.81 \pm 1.01$}\end{tabular} & \begin{tabular}[c]{@{}r@{}}$59.55 \pm 0.54$\\ {\scriptsize\textsuperscript{$\dagger$}$59.82 \pm 0.33$}\end{tabular} & \begin{tabular}[c]{@{}r@{}}$0.271 \pm 0.003$\\ {\scriptsize\textsuperscript{$\dagger$}$0.268 \pm 0.001$}\end{tabular} \\ \midrule

\textbf{GIN+IGEL}                                                                            & $0.103 \pm 0.004$                                                                                               & $86.762 \pm 0.029$                                                                                           & $78.92 \pm 0.92$                                                                                              & \multicolumn{1}{c}{---} & \multicolumn{1}{c}{---} \\ \midrule
\textbf{GIN+ELENE}                                                                                                                       & $0.092 \pm 0.001$                                                                                               & $86.783 \pm 0.044$                                                                                                & $78.92 \pm 0.35$ & $56.34 \pm 0.06$ & $0.277 \pm 0.002$                                                                                             \\
\textbf{GIN+ELENE-L (\textbf{ND})}                                                                                                                     & $0.083 \pm 0.004$                                                                                               & $86.828 \pm 0.002$                                                                                                & $78.26 \pm 0.93$ & $68.95 \pm 0.25$ & \underline{$\mathbf{0.294 \pm 0.001}$}                                                                                             \\ \midrule
\textbf{\scriptsize\begin{tabular}[c]{@{}r@{}}Best Result\\ (ELENE / ELENE-L)\end{tabular}}                                & {\ul \textbf{$\mathbf{0.079 \pm 0.003}$}}                                                                       & $86.828 \pm 0.002$                                                                                                & $79.15 \pm 1.45$ & $68.95 \pm 0.25$ & \underline{$\mathbf{0.294 \pm 0.001}$} \\ \bottomrule
\end{tabular}
\end{table}

\autoref{tab:MemoryTimePerformance} shows time and memory costs of \textsc{Elene} compared to state-of-the-art methods.  Despite not tuning hyperparameters, a GIN+\textsc{Elene-L} model outperforms GIN-AK in CIFAR and a strong GIN-AK\textsuperscript{+} baseline in MolPCBA. Furthermore, our setup of GIN layers combined with \textsc{Elene} always outperforms GNN-AK\textsuperscript{+} in terms of memory consumption. In ZINC, GIN+\textsc{Elene-L} (\textbf{ND}) requires 0.99GB compared to the 1.68GB of GIN-AK\textsuperscript{+} while reaching comparable performance ($0.083 \pm 0.004$ vs $0.082 \pm 0.003$, respectively). In MolHIV, GIN+\textsc{Elene} model requires only 70MB during training while outperforming the ROC of our reproduced run of GIN-AK\textsuperscript{+}, which required 790MB---an 11.3-fold reduction in memory usage. On PATTERN, we find that GIN+\textsc{Elene-L} (\textbf{ND}) achieves 99.95\% of the performance of GIN-AK\textsuperscript{+} while consuming only 7.8GB of memory during training, compared to the 26.52GB reported by~\citet{zhao2022}--- a 3.4-fold reduction. 

In summary, \textsc{Elene} and \textsc{Elene-L} (\textbf{ND}) achieve comparable results to the baselines with favorable time / memory efficiency. \textsc{Elene} encodings used as node-features (\textbf{GIN}+\textsc{Elene}) add minor overhead over GIN and match or outperform \textbf{GIN}+\textsc{Igel} in all tested settings, and GIN-AK in four over five. \textsc{Elene-L} (\textbf{ND}) also shows favorable memory performance versus GIN-AK and GIN-AK\textsuperscript{+} in all setups. Finally, we observe additional memory costs for \textsc{Elene-L} (\textbf{ED}) due to using node and edge embeddings.  

\begin{table}[!ht]
\centering
\caption[Memory and time performance on benchmark datasets.]{Memory and time performance on benchmark datasets, controlling for shared resource use as per~\autoref{subsec:EleneExpProtocol}. We report average epoch duration in seconds (s) and maximum memory consumption in gigabytes (GB) respectively. Dashed entries indicate executions that terminated due to running out of memory.}
\label{tab:MemoryTimePerformance}
\footnotesize
\begin{tabular}{@{}rrrrrrrrrrr@{}}
\toprule
\multicolumn{1}{l}{\multirow{1}{*}{}}       & \multicolumn{2}{c}{\textbf{ZINC}} & \multicolumn{2}{c}{\textbf{PATTERN}} & \multicolumn{2}{c}{\textbf{MolHIV}} & \multicolumn{2}{c}{\textbf{CIFAR10}} & \multicolumn{2}{c}{\textbf{MolPCBA}} \\ \midrule
\textbf{GIN}                                & $6.02$s          & $0.12$GB          & $118.62$s        & $1.42$GB          & $14.88$s         & $0.07$GB         & $98.37$s         & $0.90$GB          & $223.13$s        & $0.44$GB          \\
\textbf{GIN-AK}                             & $9.76$s          & $1.11$GB          & ---              & ---               & $19.30$s         & $0.64$GB         & $283.93$s        & $18.80$GB         & $534.78$s        & $3.80$GB          \\
\textbf{GIN-AK\textsuperscript{+}}          & $13.63$s         & $1.68$GB          & ---              & ---               & $25.47$s         & $0.79$GB         & ---              & ---               & $607.89$s        & $3.83$GB          \\ \midrule
\textbf{GIN}\textbf{+ELENE}                 & $6.14$s          & $0.13$GB          & $90.15$s         & $1.47$GB          & $14.94$s         & $0.07$GB         & $120.21$s        & $0.91$GB          & $278.29$s        & $0.46$GB          \\ 
\textbf{+ELENE-L (\textbf{ND})} & $10.23$s         & $0.99$GB          & $146.15$s        & $7.80$GB          & $42.53$s         & $0.54$GB         & $224.43$s        & $10.72$GB         & $451.12$s        & $2.39$GB          \\ 
\textbf{+ELENE-L (\textbf{ED})} & $22.61$s         & $2.85$GB          & ---              & ---               & $32.28$s         & $1.40$GB         & ---              & ---               & $1025.58$s       & $7.10$GB          \\ \bottomrule
\end{tabular}
\end{table}

\subsection{Memory Scalability}\label{subsec:ELENEMemScalability}

Finally, we evaluate the scalability of \textsc{Elene-L} in a learning setting. 
We analyze the memory consumption as a function of the graph size and how it compares with other methods.
For that, we follow a similar setting as~\cite{mitton2023subgraph} for SPEN: we design and implement a learning task on a large $d$-regular graph, and use it to explore memory consumption under different values of degree $d$ and nodes $n$. 
With this setup, we train the model for $25$ epochs to predict a constant variable so that both input tensors and gradient computations are kept in memory.

We evaluate both \textsc{Elene-L} (\textbf{ND}) and (\textbf{ED}), together with a GIN model without any \textsc{Elene-L} features, GIN-AK and GIN-AK$^{+}$ as baselines. For all \textsc{Elene-L} variants, we execute the benchmark with different values of $k \in \{1, 2, 3\}$.

\begin{figure}[!ht]
  \centering
  \includegraphics[width=1.0\linewidth]{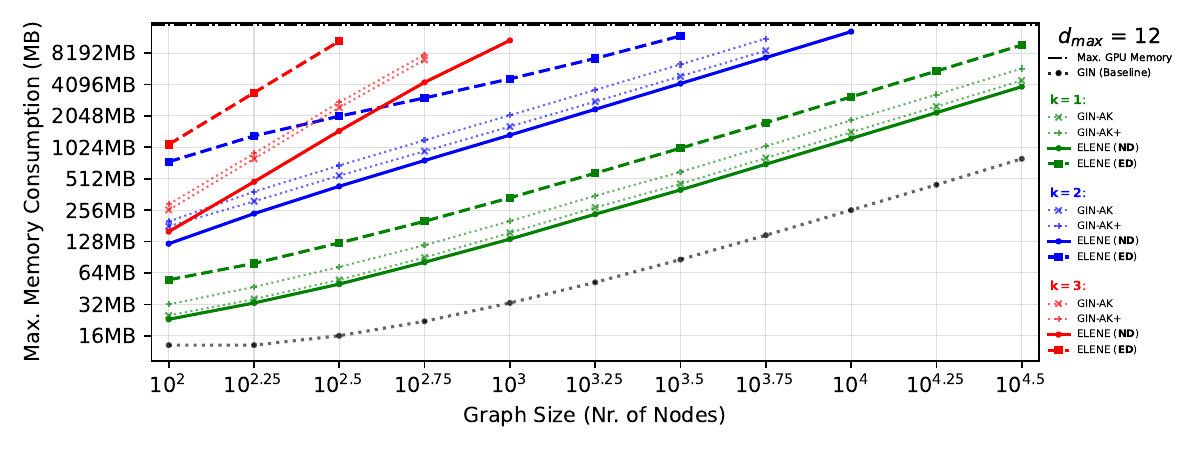}
  \caption[Memory scalability analysis of \textsc{Elene}.]{Memory scalability analysis of \textsc{Elene} when $d_\texttt{max} = 12$. We include GIN (dotted line) and maximum GPU memory (dash-and-dotted line) as indicative lower and upper memory bounds. \textsc{Elene-L} (\textbf{ND}, full lines) outperforms both GIN-AK, GIN-AK$^{+}$ (dotted lines) and \mbox{\textsc{Elene-L}} (\textbf{ED}, dashed lines). Additionally, \textsc{Elene-L} (\textbf{ND}) can encode all $d$-regular graphs in the benchmark when $k=1$. As expected, memory consumption increases linearly with the number of nodes as $d_\texttt{max}$ is kept fixed.}
  \label{fig:ELENEScalabilityPlot}
\end{figure}

\autoref{fig:ELENEScalabilityPlot} shows our results when $d_\texttt{max} = 12$.
As expected, the memory cost grows exponentially as a function of~$n$. We observe that \textsc{Elene-L} (\textbf{ND}) with $k=2$ can scale up to graphs with $10,000$ nodes with ego-network sub-graphs. Since all nodes have the same degree $d_\texttt{max} = 12$, at 2 hops we are guaranteed to find the root node, its 12 neighbors, and at least one additional neighbor at 2 hops ---or 14 total nodes. In practice, as the graphs are randomly generated, we find that each of the $2$-depth subgraphs contains an average of $144.13$ nodes with the expected maximum at $145$. 

Additionally, our experiments show that \textsc{Elene-L} (\textbf{ND}) with $k=1$ can scale up to graphs with $10^{4.5} = 31,623$ nodes with up to $d_\texttt{max} = 18$. Furthermore, despite requiring additional memory, the more expressive \textsc{Elene-L} (\textbf{ED}) can nevertheless be used for $k=1$ for graphs with up to $10^{4.5}$ nodes as well.
Scaling to graphs with more than $10^5$ nodes is possible by increasing the cpu-memory (our limit is 96GB) or alternatively changing the implementation. The latter can be done for example by computing the encodings through parallel BFS, which would result in a slower algorithm (but of the same order of complexity). 
We provide additional results for $d_\texttt{max} = 6$ and $d_\texttt{max} = 18$, as well as different graph density patterns, in~\autoref{sec:EleneDetTrainingSummary}.

Recent methods like SPEN outperform global permutation-equivariant methods like PPGN. However, SPEN still struggles to process significantly smaller graphs, with $n\approx 1,000$ nodes and $k=1$ depth ego-networks that contain 9 nodes, even with higher GPU memory as reported by~\citet{mitton2023subgraph}. 
Compared with the complexity of SPEN, which is $\mathcal{O}(n \cdot {|\mathcal{V}^{k}_v|}^2)$, \textsc{Elene-L} can encode ego-networks with $16$ times more nodes with comparable memory usage---while outperforming sub-graph GNN baselines like GIN-AK and GIN-AK\textsuperscript{+}. 

\subsection{Experiments Summary}

\textsc{Elene} and \textsc{Elene-L} consistently boost GNN performance on the three experimental benchmarks, and are shown to be scalable in~\autoref{tab:MemoryTimePerformance} and~\autoref{fig:ELENEScalabilityPlot}. 

On Expressivity, \autoref{subsec:EleneExpressivityResults} gives empirical support for~\autoref{th:EleneLMoreExpressive}, i.e. that \textsc{Elene-L} (\textbf{ED}) can distinguish \texttt{SRG}s, \emph{achieving 100\% accuracy} on the challenging SR25~\citep{balcilar2021breaking} dataset. Although SUN~\citep{frasca2022understanding} outperforms other models on Counting Substructures, \textsc{Elene} and \textsc{Elene-L} still improve baseline performance and match previous GIN-AK and GIN-AK\textsuperscript{+} baselines respectively. On Graph Properties, GIN+\textsc{Elene-L} matches existing baselines, and a GIN-AK\textsuperscript{+} model with \textsc{Elene-L} \emph{outperforms} previous state-of-the-art results on the \texttt{IsCon.} and \texttt{Diam.} tasks with -2.715 and -4.072 $\log_{10}$(MSE) each.

On $h$-Proximity, \autoref{subsec:EleneProximityResults} validates~\autoref{th:EleneLSPNNExpressive}, i.e., that \textsc{Elene-L} (\textbf{ND}) is at least as expressive as SPNNs~\citep{abboud2022spnn}, as \textsc{Elene-L} (\textbf{ND}) \emph{outperforms SPNNs and Graphormers} at capturing \emph{attributed structures}---that sparse \textsc{Elene} vectors alone \emph{cannot capture}.

On Real World Graphs from~\autoref{subsec:EleneLargeBenchmark}, \textsc{Elene} and \textsc{Elene-L} reach state-of-the-art results. On ZINC, GIN-AK\textsuperscript{+} with \textsc{Elene-L} achieves $0.079 \pm 0.003$ MAE, matching CIN~\citep{bodnar2021}. A GIN+\textsc{Elene-L} matches $99.95$\% of the performance of baselines on PATTERN while \emph{consuming 3.4$\times$ less memory}, and GIN+\textsc{Elene} reaches 0.1\% less accuracy than GIN-AK\textsuperscript{+} but does so using 1.47GB, compared to the 26.54GB reported for GIN-AK\textsuperscript{+}---\emph{a 18.1$\times$ memory reduction}. Finally, a GIN+\textsc{Elene-L} \emph{matches state-of-the-art results on MolPCBA} ($0.294 \pm 0.001$ vs $0.293 \pm 0.003$ of GIN-AK\textsuperscript{+}~\citep{zhao2022}) without hyper-parameter tuning and while \emph{consuming 37.60\% less memory (2.39GB vs 3.83GB)}. 

On Memory Scalability, \autoref{subsec:ELENEMemScalability} shows that \textsc{Elene-L} can be used on $d$-regular graphs with more than $10^4$ nodes where $d_\texttt{max} \in \{6, 12, 18\}$, validating the expected memory costs from~\autoref{subsec:ELENEAlgorithmicAnalysis} and outperforming the memory consumption of strong GIN-AK and GIN-AK\textsuperscript{+} baselines and recent methods like SPEN~\citep{mitton2023subgraph}.

\section{Conclusions}\label{sec:EleneConclusions}

We presented \textsc{Elene}, a principled edge-level ego-network encoding capturing the structural signals sufficient to distinguish 3-WL equivalent \texttt{SRG}s. We proposed two variants---\textsc{Elene} and \textsc{Elene-L}---and showed that Node-Centric and Edge-Centric representations exhibit different expressive power. To position our findings, we formally drew connections between \textsc{Elene} and recent Sub-Graph GNNs, Graph Transformers, and Shortest Path Neural Networks. 

Empirically, we evaluated our methods on 10 different tasks, where the sparse \textsc{Elene} vectors improve performance on structural expressivity tasks. Our learnable Edge-Centric \textsc{Elene-L} variant boosts MP-GNN expressivity to reach 100\% accuracy on the challenging SR25 dataset, while its Node-Centric counterpart improves over a strong baseline on the $h$-Proximity task and matches state-of-the-art results in several real-world graphs. Finally, we found our methods provide a trade-off between memory usage and structural expressivity, improving memory usage with up to 18.1$\times$ lower memory costs compared to sub-graph GNN baselines.

\subsubsection*{Broader Impact Statement}

Our main contributions are (a) a novel family of edge-aware features that can be used alone or during learning in MP-GNNs, with (b) a formal analysis of their expressivity that shows they can distinguish challenging \texttt{SRG}s, and (c) experimental results matching state-of-the-art learning models with favorable memory costs. 

We do not foresee ethical implications of our theoretical findings. Our experimental results are competitive with state-of-the-art methods at a lower memory footprint, which may help solve tasks with limited memory budgets.

\if\cameraReady1


\subsubsection*{Acknowledgements}

This work is part of the action CNS2022-136178 financed by MCIN/AEI/10.13039/501100011033 and by the EU Next Generation EU/PRTR. This work has been co-funded by MCIN/AEI/10.13039/501100011033 under the Maria de Maeztu Units of Excellence Programme (CEX2021-001195-M).
\fi

\bibliography{references}

\begin{thebibliography}{54}
\providecommand{\natexlab}[1]{#1}
\providecommand{\url}[1]{\texttt{#1}}
\expandafter\ifx\csname urlstyle\endcsname\relax
  \providecommand{\doi}[1]{doi: #1}\else
  \providecommand{\doi}{doi: \begingroup \urlstyle{rm}\Url}\fi

\bibitem[Abboud et~al.(2021)Abboud, Ceylan, Grohe, and Lukasiewicz]{abboud2021}
Ralph Abboud, \.Ismail~\.Ilkan Ceylan, Martin Grohe, and Thomas Lukasiewicz.
\newblock The surprising power of graph neural networks with random node
  initialization.
\newblock In \emph{Proceedings of the Thirtieth International Joint Conference
  on Artificial Intelligence, {IJCAI-21}}, pp.\  2112--2118, 8 2021.

\bibitem[Abboud et~al.(2022)Abboud, Dimitrov, and Ceylan]{abboud2022spnn}
Ralph Abboud, Radoslav Dimitrov, and {\.I}smail~{\.I}lkan Ceylan.
\newblock Shortest path networks for graph property prediction.
\newblock In \emph{Proceedings of the First Learning on Graphs Conference
  ({LoG})}, 2022.

\bibitem[Alon \& Yahav(2021)Alon and Yahav]{alon2021on}
Uri Alon and Eran Yahav.
\newblock On the bottleneck of graph neural networks and its practical
  implications.
\newblock In \emph{International Conference on Learning Representations}, 2021.

\bibitem[Alvarez-Gonzalez et~al.(2022)Alvarez-Gonzalez, Kaltenbrunner, and
  G{\'o}mez]{alvarez-gonzalez2022beyond}
Nurudin Alvarez-Gonzalez, Andreas Kaltenbrunner, and Vicen{\c{c}} G{\'o}mez.
\newblock Beyond 1-{WL} with local ego-network encodings.
\newblock In \emph{The First Learning on Graphs Conference}, 2022.

\bibitem[Arvind et~al.(2020)Arvind, Fuhlbrück, Köbler, and
  Verbitsky]{arvind2020}
V.~Arvind, Frank Fuhlbrück, Johannes Köbler, and Oleg Verbitsky.
\newblock On weisfeiler-leman invariance: Subgraph counts and related graph
  properties.
\newblock \emph{Journal of Computer and System Sciences}, 113:\penalty0 42--59,
  2020.

\bibitem[Babai \& Kucera(1979)Babai and Kucera]{babai1979}
Laszlo Babai and Ludik Kucera.
\newblock Canonical labelling of graphs in linear average time.
\newblock In \emph{20th Annual Symposium on Foundations of Computer Science
  (sfcs 1979)}, pp.\  39--46, 1979.
\newblock \doi{10.1109/SFCS.1979.8}.

\bibitem[Balcilar et~al.(2021)Balcilar, H{\'e}roux, Ga{\"u}z{\`e}re, Vasseur,
  Adam, and Honeine]{balcilar2021breaking}
Muhammet Balcilar, Pierre H{\'e}roux, Benoit Ga{\"u}z{\`e}re, Pascal Vasseur,
  S{\'e}bastien Adam, and Paul Honeine.
\newblock Breaking the limits of message passing graph neural networks.
\newblock In \emph{Proceedings of the 38th International Conference on Machine
  Learning (ICML)}, 2021.

\bibitem[Barabási \& Albert(1999)Barabási and
  Albert]{barabasi-albert1999prefattach}
Albert-László Barabási and Réka Albert.
\newblock Emergence of scaling in random networks.
\newblock \emph{Science}, 286\penalty0 (5439):\penalty0 509--512, 1999.
\newblock \doi{10.1126/science.286.5439.509}.

\bibitem[Barceló et~al.(2020)Barceló, Kostylev, Monet, Pérez, Reutter, and
  Silva]{barceló2020logic}
Pablo Barceló, Egor~V. Kostylev, Mikael Monet, Jorge Pérez, Juan Reutter, and
  Juan~Pablo Silva.
\newblock The logical expressiveness of graph neural networks.
\newblock In \emph{International Conference on Learning Representations}, 2020.

\bibitem[Battaglia et~al.(2016)Battaglia, Pascanu, Lai, Jimenez~Rezende, and
  kavukcuoglu]{battaglia_NIPS2016}
Peter Battaglia, Razvan Pascanu, Matthew Lai, Danilo Jimenez~Rezende, and koray
  kavukcuoglu.
\newblock Interaction networks for learning about objects, relations and
  physics.
\newblock In \emph{Advances in Neural Information Processing Systems},
  volume~29, 2016.

\bibitem[Bevilacqua et~al.(2022)Bevilacqua, Frasca, Lim, Srinivasan, Cai,
  Balamurugan, Bronstein, and Maron]{bevilacqua2022equivariant}
Beatrice Bevilacqua, Fabrizio Frasca, Derek Lim, Balasubramaniam Srinivasan,
  Chen Cai, Gopinath Balamurugan, Michael~M. Bronstein, and Haggai Maron.
\newblock Equivariant subgraph aggregation networks.
\newblock In \emph{International Conference on Learning Representations}, 2022.

\bibitem[Bevilacqua et~al.(2023)Bevilacqua, Eliasof, Meirom, Ribeiro, and
  Maron]{bevilacqua2023efficient}
Beatrice Bevilacqua, Moshe Eliasof, Eli Meirom, Bruno Ribeiro, and Haggai
  Maron.
\newblock Efficient subgraph gnns by learning effective selection policies.
\newblock \emph{arXiv preprint arXiv:2310.20082}, 2023.

\bibitem[Bodnar et~al.(2021)Bodnar, Frasca, Otter, Wang, Li\`{o}, Montufar, and
  Bronstein]{bodnar2021}
Cristian Bodnar, Fabrizio Frasca, Nina Otter, Yuguang Wang, Pietro Li\`{o},
  Guido~F Montufar, and Michael Bronstein.
\newblock Weisfeiler and {L}ehman go cellular: {CW} networks.
\newblock In \emph{Advances in Neural Information Processing Systems},
  volume~34, 2021.

\bibitem[Bouritsas et~al.(2023)Bouritsas, Frasca, Zafeiriou, and
  Bronstein]{bouritsas2023}
Giorgos Bouritsas, Fabrizio Frasca, Stefanos Zafeiriou, and Michael~M.
  Bronstein.
\newblock Improving graph neural network expressivity via subgraph isomorphism
  counting.
\newblock \emph{IEEE Transactions on Pattern Analysis and Machine
  Intelligence}, 45\penalty0 (1):\penalty0 657--668, 2023.

\bibitem[Brouwer \& Van~Maldeghem(2022)Brouwer and Van~Maldeghem]{srg2022}
Andries~E. Brouwer and Hendrik Van~Maldeghem.
\newblock \emph{Strongly regular graphs}, volume 182.
\newblock Cambridge University Press, 2022.
\newblock ISBN 9781316512036.

\bibitem[Chen et~al.(2020)Chen, Chen, Villar, and
  Bruna]{chen2022CanGNNsCountSubstructures}
Zhengdao Chen, Lei Chen, Soledad Villar, and Joan Bruna.
\newblock Can graph neural networks count substructures?
\newblock In \emph{Advances in Neural Information Processing Systems},
  volume~33, pp.\  10383--10395, 2020.

\bibitem[Corso et~al.(2020)Corso, Cavalleri, Beaini, Li\`{o}, and
  Velickovic]{corso2020PNAGraphProp}
Gabriele Corso, Luca Cavalleri, Dominique Beaini, Pietro Li\`{o}, and Petar
  Velickovic.
\newblock Principal neighbourhood aggregation for graph nets.
\newblock In \emph{Proceedings of the 34th International Conference on Neural
  Information Processing Systems}, NIPS'20, Red Hook, NY, USA, 2020.
\newblock ISBN 9781713829546.

\bibitem[Duvenaud et~al.(2015)Duvenaud, Maclaurin, Iparraguirre, Bombarell,
  Hirzel, Aspuru-Guzik, and Adams]{DuvenaudNIPS2015}
David~K Duvenaud, Dougal Maclaurin, Jorge Iparraguirre, Rafael Bombarell,
  Timothy Hirzel, Alan Aspuru-Guzik, and Ryan~P Adams.
\newblock Convolutional networks on graphs for learning molecular fingerprints.
\newblock In \emph{Advances in Neural Information Processing Systems},
  volume~28, 2015.

\bibitem[Dwivedi et~al.(2020)Dwivedi, Joshi, Laurent, Bengio, and
  Bresson]{dwivedi2020benchmarkgnns}
Vijay~Prakash Dwivedi, Chaitanya~K Joshi, Thomas Laurent, Yoshua Bengio, and
  Xavier Bresson.
\newblock Benchmarking graph neural networks.
\newblock \emph{arXiv preprint arXiv:2003.00982}, 2020.

\bibitem[Dwivedi et~al.(2022)Dwivedi, Luu, Laurent, Bengio, and
  Bresson]{dwivedi2022graph}
Vijay~Prakash Dwivedi, Anh~Tuan Luu, Thomas Laurent, Yoshua Bengio, and Xavier
  Bresson.
\newblock Graph neural networks with learnable structural and positional
  representations.
\newblock In \emph{International Conference on Learning Representations}, 2022.

\bibitem[Eliasof et~al.(2022)Eliasof, Haber, and Treister]{eliasof2022pathgcn}
Moshe Eliasof, Eldad Haber, and Eran Treister.
\newblock path{GCN}: Learning general graph spatial operators from paths.
\newblock In Kamalika Chaudhuri, Stefanie Jegelka, Le~Song, Csaba Szepesvari,
  Gang Niu, and Sivan Sabato (eds.), \emph{Proceedings of the 39th
  International Conference on Machine Learning}, volume 162 of
  \emph{Proceedings of Machine Learning Research}, pp.\  5878--5891. PMLR,
  17--23 Jul 2022.

\bibitem[Frasca et~al.(2022)Frasca, Bevilacqua, Bronstein, and
  Maron]{frasca2022understanding}
Fabrizio Frasca, Beatrice Bevilacqua, Michael~M Bronstein, and Haggai Maron.
\newblock Understanding and extending subgraph gnns by rethinking their
  symmetries.
\newblock In \emph{Advances in Neural Information Processing Systems}, 2022.

\bibitem[Gilmer et~al.(2017)Gilmer, Schoenholz, Riley, Vinyals, and
  Dahl]{gilmer2017neural}
Justin Gilmer, Samuel~S. Schoenholz, Patrick~F. Riley, Oriol Vinyals, and
  George~E. Dahl.
\newblock Neural message passing for quantum chemistry.
\newblock In \emph{Proceedings of the 34th International Conference on Machine
  Learning (ICML)}, volume~70, 2017.

\bibitem[Grohe(2021)]{grohe2021logic}
Martin Grohe.
\newblock The logic of graph neural networks.
\newblock In \emph{Proceedings of the 36th Annual ACM/IEEE Symposium on Logic
  in Computer Science}, LICS '21, New York, NY, USA, 2021. Association for
  Computing Machinery.
\newblock ISBN 9781665448956.
\newblock \doi{10.1109/LICS52264.2021.9470677}.

\bibitem[Gutteridge et~al.(2023)Gutteridge, Dong, Bronstein, and
  Di~Giovanni]{gutteridge2023drew}
Benjamin Gutteridge, Xiaowen Dong, Michael~M Bronstein, and Francesco
  Di~Giovanni.
\newblock {DRew}: Dynamically rewired message passing with delay.
\newblock In \emph{International Conference on Machine Learning}, pp.\
  12252--12267. PMLR, 2023.

\bibitem[Hu et~al.(2020{\natexlab{a}})Hu, Fey, Zitnik, Dong, Ren, Liu, Catasta,
  and Leskovec]{hu2020ogb}
Weihua Hu, Matthias Fey, Marinka Zitnik, Yuxiao Dong, Hongyu Ren, Bowen Liu,
  Michele Catasta, and Jure Leskovec.
\newblock Open graph benchmark: Datasets for machine learning on graphs.
\newblock \emph{arXiv preprint arXiv:2005.00687}, 2020{\natexlab{a}}.

\bibitem[Hu et~al.(2020{\natexlab{b}})Hu, Liu, Gomes, Zitnik, Liang, Pande, and
  Leskovec]{hu2020pretraining}
Weihua Hu, Bowen Liu, Joseph Gomes, Marinka Zitnik, Percy Liang, Vijay Pande,
  and Jure Leskovec.
\newblock Strategies for pre-training graph neural networks.
\newblock In \emph{International Conference on Learning Representations},
  2020{\natexlab{b}}.

\bibitem[Kim et~al.(2022)Kim, Nguyen, Min, Cho, Lee, Lee, and
  Hong]{kim2022pure}
Jinwoo Kim, Dat~Tien Nguyen, Seonwoo Min, Sungjun Cho, Moontae Lee, Honglak
  Lee, and Seunghoon Hong.
\newblock Pure transformers are powerful graph learners.
\newblock In \emph{Advances in Neural Information Processing Systems}, 2022.

\bibitem[Kipf \& Welling(2017)Kipf and
  Welling]{DBLP:journals/corr/KipfW16-SemiSup-GCN}
Thomas~N. Kipf and Max Welling.
\newblock Semi-supervised classification with graph convolutional networks.
\newblock In \emph{5th International Conference on Learning Representations,
  {ICLR}}, 2017.

\bibitem[Kong et~al.(2023)Kong, Feng, Liu, Tao, Chen, and
  Zhang]{kong2023maggnn}
Lecheng Kong, Jiarui Feng, Hao Liu, Dacheng Tao, Yixin Chen, and Muhan Zhang.
\newblock {MAG}-{GNN}: Reinforcement learning boosted graph neural network.
\newblock In \emph{Thirty-seventh Conference on Neural Information Processing
  Systems}, 2023.

\bibitem[Li et~al.(2020)Li, Wang, Wang, and Leskovec]{DEA_GNNs}
Pan Li, Yanbang Wang, Hongwei Wang, and Jure Leskovec.
\newblock Distance encoding: Design provably more powerful neural networks for
  graph representation learning.
\newblock In \emph{Advances in Neural Information Processing Systems},
  volume~33, pp.\  4465--4478, 2020.

\bibitem[Maron et~al.(2019)Maron, Ben-Hamu, Serviansky, and
  Lipman]{HaggaiNIPS2019PPGN}
Haggai Maron, Heli Ben-Hamu, Hadar Serviansky, and Yaron Lipman.
\newblock Provably powerful graph networks.
\newblock In \emph{Advances in Neural Information Processing Systems},
  volume~32, 2019.

\bibitem[Michel et~al.(2023)Michel, Nikolentzos, Lutzeyer, and
  Vazirgiannis]{michel2023pathnns}
Gaspard Michel, Giannis Nikolentzos, Johannes Lutzeyer, and Michalis
  Vazirgiannis.
\newblock Path neural networks: Expressive and accurate graph neural networks.
\newblock In \emph{Proceedings of the 40th {International} {Conference} on
  {Machine} {Learning} ({ICML})}, 2023.

\bibitem[Mitton \& Murray-Smith(2023)Mitton and
  Murray-Smith]{mitton2023subgraph}
Joshua Mitton and Roderick Murray-Smith.
\newblock Subgraph permutation equivariant networks.
\newblock \emph{Transactions on Machine Learning Research}, 2023.
\newblock ISSN 2835-8856.

\bibitem[Morris et~al.(2019)Morris, Ritzert, Fey, Hamilton, Lenssen, Rattan,
  and Grohe]{morris2019}
Christopher Morris, Martin Ritzert, Matthias Fey, William~L. Hamilton, Jan~Eric
  Lenssen, Gaurav Rattan, and Martin Grohe.
\newblock Weisfeiler and leman go neural: Higher-order graph neural networks.
\newblock \emph{Proceedings of the AAAI Conference on Artificial Intelligence},
  33\penalty0 (01):\penalty0 4602--4609, Jul. 2019.

\bibitem[Morris et~al.(2023)Morris, Lipman, Maron, Rieck, Kriege, Grohe, Fey,
  and Borgwardt]{morris2023}
Christopher Morris, Yaron Lipman, Haggai Maron, Bastian Rieck, Nils~M. Kriege,
  Martin Grohe, Matthias Fey, and Karsten Borgwardt.
\newblock Weisfeiler and leman go machine learning: the story so far.
\newblock \emph{J. Mach. Learn. Res.}, 24\penalty0 (1), mar 2023.
\newblock ISSN 1532-4435.

\bibitem[Nikolentzos et~al.(2020)Nikolentzos, Dasoulas, and
  Vazirgiannis]{nikolentzos2020khop}
Giannis Nikolentzos, George Dasoulas, and Michalis Vazirgiannis.
\newblock k-hop graph neural networks.
\newblock \emph{Neural Networks}, 130:\penalty0 195--205, 2020.

\bibitem[Oono \& Suzuki(2022)Oono and Suzuki]{oono2019graph}
Kenta Oono and Taiji Suzuki.
\newblock Graph neural networks exponentially lose expressive power for node
  classification.
\newblock In \emph{International Conference on Learning Representations}, 2022.

\bibitem[Papp \& Wattenhofer(2022)Papp and Wattenhofer]{pmlr-v162-papp22a}
P{\'a}l~Andr{\'a}s Papp and Roger Wattenhofer.
\newblock A theoretical comparison of graph neural network extensions.
\newblock In Kamalika Chaudhuri, Stefanie Jegelka, Le~Song, Csaba Szepesvari,
  Gang Niu, and Sivan Sabato (eds.), \emph{Proceedings of the 39th
  International Conference on Machine Learning}, volume 162 of
  \emph{Proceedings of Machine Learning Research}, pp.\  17323--17345. PMLR,
  17--23 Jul 2022.

\bibitem[Papp et~al.(2021)Papp, Martinkus, Faber, and
  Wattenhofer]{papp2021dropgnn}
P{\'a}l~Andr{\'a}s Papp, Karolis Martinkus, Lukas Faber, and Roger Wattenhofer.
\newblock {DropGNN:} random dropouts increase the expressiveness of graph
  neural networks.
\newblock In \emph{35th Conference on Neural Information Processing Systems},
  2021.

\bibitem[Qian et~al.(2022)Qian, Rattan, Geerts, Niepert, and
  Morris]{qian2022ordered}
Chendi Qian, Gaurav Rattan, Floris Geerts, Mathias Niepert, and Christopher
  Morris.
\newblock Ordered subgraph aggregation networks.
\newblock In Alice~H. Oh, Alekh Agarwal, Danielle Belgrave, and Kyunghyun Cho
  (eds.), \emph{Advances in Neural Information Processing Systems}, 2022.

\bibitem[Ramp\'{a}\v{s}ek et~al.(2022)Ramp\'{a}\v{s}ek, Galkin, Dwivedi, Luu,
  Wolf, and Beaini]{rampasek2022GPS}
Ladislav Ramp\'{a}\v{s}ek, Mikhail Galkin, Vijay~Prakash Dwivedi, Anh~Tuan Luu,
  Guy Wolf, and Dominique Beaini.
\newblock {Recipe for a General, Powerful, Scalable Graph Transformer}.
\newblock \emph{Advances in Neural Information Processing Systems}, 35, 2022.

\bibitem[Vignac et~al.(2020)Vignac, Loukas, and Frossard]{vignac2020smp}
Clément Vignac, Andreas Loukas, and Pascal Frossard.
\newblock Building powerful and equivariant graph neural networks with
  structural message-passing.
\newblock In H.~Larochelle, M.~Ranzato, R.~Hadsell, M.~F. Balcan, and H.~Lin
  (eds.), \emph{Advances in Neural Information Processing Systems}, volume~33,
  2020.

\bibitem[Weisfeiler \& Leman(1968)Weisfeiler and Leman]{weisfeilerlemanl1968wl}
B~Weisfeiler and A~Leman.
\newblock The reduction of a graph to canonical form and the algebra which
  appears therein.
\newblock \emph{Nauchno-Technicheskaya Informatsia}, 2(9):\penalty0 12--16,
  1968.

\bibitem[Xu et~al.(2019)Xu, Hu, Leskovec, and Jegelka]{xu2018how}
Keyulu Xu, Weihua Hu, Jure Leskovec, and Stefanie Jegelka.
\newblock How powerful are graph neural networks?
\newblock In \emph{International Conference on Learning Representations}, 2019.

\bibitem[Yan et~al.(2023)Yan, Zhou, Gao, Tang, and Zhang]{yan2023efficiently}
Zuoyu Yan, Junru Zhou, Liangcai Gao, Zhi Tang, and Muhan Zhang.
\newblock Efficiently counting substructures by subgraph gnns without running
  gnn on subgraphs.
\newblock \emph{arXiv preprint arXiv:2303.10576}, 2023.

\bibitem[Ying et~al.(2021)Ying, Cai, Luo, Zheng, Ke, He, Shen, and
  Liu]{ying2021do}
Chengxuan Ying, Tianle Cai, Shengjie Luo, Shuxin Zheng, Guolin Ke, Di~He,
  Yanming Shen, and Tie-Yan Liu.
\newblock Do transformers really perform badly for graph representation?
\newblock In \emph{35th Conference on Neural Information Processing Systems},
  2021.

\bibitem[Ying et~al.(2018)Ying, He, Chen, Eksombatchai, Hamilton, and
  Leskovec]{Ying-2018-PinSAGE}
Rex Ying, Ruining He, Kaifeng Chen, Pong Eksombatchai, William~L Hamilton, and
  Jure Leskovec.
\newblock Graph convolutional neural networks for web-scale recommender
  systems.
\newblock In \emph{Proceedings of the 24th ACM International Conference on
  Knowledge Discovery \& Data Mining}, pp.\  974--983, 2018.

\bibitem[You et~al.(2019)You, Ying, and Leskovec]{pmlr-v97-you19b}
Jiaxuan You, Rex Ying, and Jure Leskovec.
\newblock Position-aware graph neural networks.
\newblock In \emph{Proceedings of the 36th International Conference on Machine
  Learning}, volume~97 of \emph{Proceedings of Machine Learning Research}, pp.\
   7134--7143, Long Beach, California, USA, 09--15 Jun 2019. PMLR.

\bibitem[You et~al.(2021)You, Gomes-Selman, Ying, and
  Leskovec]{you2021identity}
Jiaxuan You, Jonathan~M Gomes-Selman, Rex Ying, and Jure Leskovec.
\newblock Identity-aware graph neural networks.
\newblock In \emph{35th AAAI Conference on Artificial Intelligence}, volume~35,
  pp.\  10737--10745, 2021.

\bibitem[Yun et~al.(2019)Yun, Jeong, Kim, Kang, and Kim]{Yun2019}
Seongjun Yun, Minbyul Jeong, Raehyun Kim, Jaewoo Kang, and Hyunwoo~J Kim.
\newblock Graph transformer networks.
\newblock In \emph{Advances in Neural Information Processing Systems},
  volume~32, 2019.

\bibitem[Zhang et~al.(2023)Zhang, Luo, Wang, and He]{zhang2023rethinking}
Bohang Zhang, Shengjie Luo, Liwei Wang, and Di~He.
\newblock Rethinking the expressive power of {GNN}s via graph biconnectivity.
\newblock In \emph{International Conference on Learning Representations}, 2023.

\bibitem[Zhang \& Li(2021)Zhang and Li]{zhang2021}
Muhan Zhang and Pan Li.
\newblock Nested graph neural networks.
\newblock \emph{Advances in Neural Information Processing Systems}, 34, 2021.

\bibitem[Zhao et~al.(2022)Zhao, Jin, Akoglu, and Shah]{zhao2022}
Lingxiao Zhao, Wei Jin, Leman Akoglu, and Neil Shah.
\newblock From stars to subgraphs: Uplifting any {GNN} with local structure
  awareness.
\newblock In \emph{International Conference on Learning Representations}, 2022.

\end{thebibliography}
\bibliographystyle{tmlr}

\clearpage

\appendix

\section{ELENE through BFS}
\label{sec:EleneBFS}

This appendix showcases a BFS implementation of the \textsc{Elene} Encoding that spans edges until reaching the maximum encoding depth $k$ given a node $v$ in $V$. As noted in~\autoref{subsec:ELENEAlgorithmicAnalysis}, this implementation can be trivially parallelized over $p$ processors as the encoding of each $v \in V$ is independent of other nodes.

\begin{algorithm}[hbt!]
\caption{\textsc{Elene} Node Encoding using BFS.}\label{alg:ELENEBFSEncoding}
\begin{algorithmic}[1]
    \Require{$G = (V, E), v \in V, k: \mathbb{N}$}
    \State $\texttt{distances} := \{v: 0\}$ \Comment{Mapping of nodes to their distance to $n$, i.e. $l_G(u, v)$.}
    \State $\texttt{r\_degrees} := \{v: (0, 0, 0)\}$ \Comment{Mapping of nodes to their relative degrees, i.e. $d_{\mathcal{S}}^{(p)}(u|v)$.}
    \For{$(\texttt{src}, \texttt{dst}) \text{ in } G.\texttt{bfs\_edges}(n, \texttt{max\_depth}=k)$}
        \If {$\texttt{src} \notin \texttt{distances}$} \Comment{Invariant: only one node can be unknown / not in $\texttt{distances}$.}
            \State{$\texttt{dst}, \texttt{src} := \texttt{src}, \texttt{dst}$}
        \EndIf
        \If {$\texttt{dst} \notin \texttt{distances}$} \Comment{$\texttt{dst}$ is unknown, so its distance is one-hop after $\texttt{src}$'s.}
            \State $\texttt{distances}[\texttt{dst}] := \texttt{distances}[\texttt{src}]$
        \EndIf
        \State $\texttt{dist\_delta} := \texttt{distances}[\texttt{dst}] - \texttt{distances}[\texttt{src}]$ \Comment{Compute the distance delta in \{-1, 0, 1\}.} \\

        \LineComment{Access the relative degree counts of each node.}
        \State $\texttt{src\_deg} := \texttt{r\_degrees}.\texttt{get} (\texttt{src}, [0, 0, 0])$ 
        \State $\texttt{dst\_deg} := \texttt{r\_degrees}.\texttt{get}(\texttt{dst}, [0, 0, 0])$ \\ 
        
        \LineComment{Increment degree counts for each node in their respective `direction'.}
        \State $\texttt{src\_deg}[\texttt{dist\_delta} + 1]\texttt{++}$ \Comment{The indexing maps \{-1, 0, 1\} deltas into \{0, 1, 2\} vector indexes.}  
        \State $\texttt{dst\_deg}[1 - \texttt{dist\_delta}]\texttt{++}$  \\ 
        
        \LineComment{Update the relative degrees of $\texttt{src}$ and $\texttt{dst}$.}
        \State $\texttt{r\_degrees}[\texttt{src}] := \texttt{src\_deg}$ 
        \State $\texttt{r\_degrees}[\texttt{dst}] := \texttt{dst\_deg}$ 
    \EndFor \\
    \LineComment{For each $u \in \mathcal{V}^{k}_v$, compute~\autoref{eq:ELENEEncRaw} quadruplets, count their frequencies and return the mapping.}
    \State $\texttt{mapping} := \{\}$
    \For{$u \in \mathcal{V}^{k}_v$}
        \State $\texttt{quadruplet} := (\texttt{distances}[u], \texttt{r\_degrees}[u][0], \texttt{r\_degrees}[u].\texttt{sum}(), \texttt{r\_degrees}[u][2])$
        \State $\texttt{mapping}[\texttt{quadruplet}]\texttt{++}$
    \EndFor
    \Ensure{$\texttt{mapping}$}
\end{algorithmic}
\end{algorithm}

\section{Benchmark Details}
\label{sec:EleneBenchmarkDetails}

In this appendix, we provide an overview of the benchmark we execute when evaluating \textsc{Elene}, including the variants of models we test, and descriptions of our code and compute environment. We also summarize the datasets we use in~\autoref{sec:EleneDatasetsOverview}.

\textbf{Benchmark Configuration.} We build on top of the implementation from~\cite{zhao2022}, introducing explicit ego-network attributes on their evaluation framework for consistency. 

All \textsc{Elene} results are reported by extending the node and edge attributes as input into a GIN~\cite{xu2018how} extended to support edge-level features when available~\cite{hu2020pretraining}. In all experiments, we evaluate \textsc{Elene-L} on top of GINs with edge extensions~\cite{hu2020pretraining}. 

For all explicit ego-network attribute methods, we summarize the available hyper-parameters in~\autoref{tab:EgoNetAttrHyperParams}. For the implementation of \textsc{Elene-L}, we observed unstable training when using the sum pooling function during early stages of development. We found that training was stable using masked \texttt{Mean} pooling where the $n$ node messages (or $m$ for edge messages) in the ego-network sub-graph are averaged considering a binary mask for neighbors of the root node at a distance $k$ or less. All our results are reported using \texttt{Mean} pooling, including our results on SR25, suggesting that this decision does not adversely impact the model expressivity expected from~\autoref{sec:EleneExpressivity}. The resulting implementation of~\autoref{eq:ELENE-L-NodeMessage} is: 
\begin{align*}
    \Phi_{\texttt{ND}}^{t}(v) = \Phi_{\texttt{out}}\Bigg(\textbf{x}_v^{t} \biggconcat \sum_{u}^{\mathcal{V}^{k}_v} \frac{\Phi_{\texttt{nd}}^{t}(u|v)}{\texttt{size}(\mathcal{V}^{k}_v)} \biggconcat \sum_{\langle u, w\rangle}^{\mathcal{E}^{k}_v} \frac{\Phi_{\texttt{ed}}^{t}(u, w|v)}{\texttt{size}(\mathcal{E}^{k}_v)}\Bigg).
\end{align*}
We use an analogous implementation for~\autoref{eq:ELENE-L-Edge}. Additionally, in our experimental benchmark we choose to implement \textsc{Elene-L} (\textbf{ND}) without the $\Phi_{\texttt{ed}}^{t}$ term so that it more closely follows the node-centric \textsc{Elene} encodings of~\autoref{eq:ELENEEncRaw}, reducing memory costs. We describe the hyper-parameters implemented to control our models in~\autoref{tab:EgoNetAttrHyperParams}. 

\begin{table}[!ht]
\centering
\caption[\textsc{Elene} Hyper-parameters.]{Hyper-parameters controlling the behaviour of explicit ego-network attribute encodings. \textsc{Elene} only relies on $k$, while \textsc{Elene-L} has 5 additional configurable settings.}
\label{tab:EgoNetAttrHyperParams}
\small
\begin{tabular}{@{}cccc@{}}
\toprule
\textbf{Parameter}              & \textbf{\textsc{Elene}}                                                                      & \textbf{\textsc{Elene-L}}                                                                                                         \\ \midrule
\textbf{Depth of Ego-Net ($k$)} & \{0, 1, 2\}                                                                                  & \{0, 1, 2, 3\}                                                                                                                  \\ \midrule
\textbf{Embedding Type}         & Sparse                                                                                       & Dense, learned                                                                                               \\ \midrule
\textbf{Representation}         & Node-only                                                                                    & Node-centric (\textbf{ND}), Edge-centric (\textbf{ED})                                                                                                 \\ \midrule
\textbf{Max. Encoded Degree}    & \begin{tabular}[c]{@{}c@{}}Set to $d_\texttt{max}$\\ from the training dataset.\end{tabular} & \begin{tabular}[c]{@{}c@{}}Set to $d_\texttt{max}$\\ from the training dataset or 0 \\ (ignore degree info).\end{tabular}       \\ \midrule
\textbf{Max. Encoded Distance}  & Equal to $k$                                                                                 & \begin{tabular}[c]{@{}c@{}}Set to $k$. \\ Can be modified to control \\ the sub-graph mean norm. factor.\end{tabular}           \\ \bottomrule
\end{tabular}
\end{table}

\textbf{Tested Models.} On the Expressivity tasks, ZINC and MolHIV, we evaluate all learnable variants (\textbf{ND} and \textbf{ED}), while on the remaining classification/regression benchmarks we only consider \textbf{(ND)} models due to reduced memory costs and limited computational bandwidth. Furthermore, in all \textsc{Elene-L} setups we only test a reduced number of hyper-parameters due to computational constraints, unless specified otherwise, only evaluating different values of the maximum sub-graph distance to embed. We describe the hyper-parameters and modeling choices in detail in~\autoref{sec:EleneDetTrainingSummary}.

\subsection{Dataset Details}\label{sec:EleneDatasetsOverview}

We summarize the key aspects of the datasets we use to evaluate our proposed methods in~\autoref{sec:EleneExperiments}. \autoref{tab:DatasetStats} contains an overview of each benchmark and dataset, the objective being addressed, and high-level dataset statistics---namely number of graphs, average number of nodes ($n$) and edges ($m$) per graph.

\begin{table}
\centering
\caption{Dataset statistics.}
\label{tab:DatasetStats}
\scriptsize
\begin{tabular}{@{}cclrrrr@{}}
\toprule
\textbf{Benchmark}                                                                    & \textbf{Dataset}       & \multicolumn{1}{c}{\textbf{Objective}} & \multicolumn{1}{c}{\textbf{Tasks}} & \multicolumn{1}{c}{\textbf{\begin{tabular}[c]{@{}c@{}}Nr. of Graphs\\ (Train / Valid / Test)\end{tabular}}} & \multicolumn{1}{c}{\textbf{Avg. $n$}} & \multicolumn{1}{c}{\textbf{Avg. $m$}} \\ \midrule
\multicolumn{1}{l}{\multirow{4}{*}{\textbf{Expressivity}}}                            & \textbf{EXP}           & Distinguish 1-WL Equiv. graphs         & 2                                  & 1200                                                                                                   & 44.4                                  & 110.2                                 \\
\multicolumn{1}{l}{}                                                                  & \textbf{SR25}          & Distinguish 3-WL Equiv. graphs         & 15                                 & 15                                                                                                     & 25                                    & 300                                   \\
\multicolumn{1}{l}{}                                                                  & \textbf{CountingSub.}  & Count graph substructures              & 4                                  & 1500 / 1000 / 2500                                                                                     & 18.8                                  & 62.6                                  \\
\multicolumn{1}{l}{}                                                                  & \textbf{GraphProp.}    & Regress graph properties               & 3                                  & 5120 / 640 / 1280                                                                                      & 19.5                                  & 101.1                                 \\ \midrule
\multirow{5}{*}{\textbf{\begin{tabular}[c]{@{}c@{}}Real World\\ Graphs\end{tabular}}} & \textbf{ZINC-12K}      & Molecular prop. regression             & 1                                  & 10000 / 1000 / 1000                                                                                    & 23.1                                  & 49.8                                  \\
                                                                                      & \textbf{CIFAR10}       & Multi-class class.                     & 10                                 & 45000 / 5000 / 10000                                                                                   & 117.6                                 & 1129.8                                \\
                                                                                      & \textbf{PATTERN}       & Recognize subgraphs                    & 2                                  & 10000 / 2000 / 2000                                                                                    & 118.9                                 & 6079.8                                \\
                                                                                      & \textbf{MolHIV}        & Binary class.                          & 1                                  & 32901 / 4113 / 4113                                                                                    & 25.5                                  & 54.1                                  \\
                                                                                      & \textbf{MolPCBA}       & Multi-label binary class.              & 128                                & 350343 / 43793 / 43793                                                                                 & 25.6                                  & 55.4                                  \\ \midrule
\textbf{Proximity}                                                                    & \textbf{$h$-Proximity} & Binary classification                  & 4                                  & 9000                                                                                                   & 117.14                                & 1484.82                               \\ \bottomrule
\end{tabular}
\end{table}

\subsection{Detailed Experimental Summary}\label{sec:EleneDetTrainingSummary}

In this section, we summarize our experimental setup and training procedure, describing the hyper-parameters that we consider in each setting. For all the experiments described, we evaluate the \textsc{Elene} encodings by concatenating them with the node feature vectors and as part of the edge features when available, using the element-wise product following the same approach as \textsc{Igel}. 

\textbf{Expressivity.} See \texttt{expressivityDatasets.sh} for details.\\

\emph{---EXP and SR25.} We evaluate \textsc{Elene} on GIN and GIN-AK\textsuperscript{+} for both data sets with $k \in \{0, 1, 2\}$. For \textsc{Elene-L}, we evaluate all model variants for $k \in \{0, 1, 2\}$ with 8-dim embeddings for EXP and 32-dim embeddings for SR25. All models use $L = 4$ for EXP and $L = 2$ for SR25.

\emph{---Counting Sub. and Graph Prop.} We evaluate \textsc{Elene} on GIN and GIN-AK\textsuperscript{+} for both data sets with $k \in \{0, 1, 2\}$. For \textsc{Elene-L}, we evaluate all model variants for $k \in \{0, 1, 2\}$ with 16-dim embeddings. On the GraphProp dataset, we additionally try $k = 3$ after noticing expected positive results during early evaluation---as larger values of $k$ enable the model to capture long-range dependencies. All models use $L = 3$ for Counting Sub. and $L = 6$ for Graph Prop.\\

\textbf{Real World Graphs.} See \texttt{benchmarkDatasets.sh} for details.\\

\emph{---ZINC and MolHIV.} We evaluate \textsc{Elene} on GIN and GIN-AK\textsuperscript{+} for both data sets with $k \in \{0, 1, 2\}$. For \textsc{Elene-L}, we evaluate all model variants for $k \in \{0, 1, 2, 3\}$ with 32-dim embeddings. All models use $L = 6$ for ZINC and $L = 2$ for MolHIV.

\emph{---PATTERN.} We evaluate \textsc{Elene} on GIN with $k \in \{0, 1, 2\}$ and on GIN-AK\textsuperscript{+} $k \in \{0, 1\}$. For \textsc{Elene-L}, we evaluate all model variants for $k \in \{0, 1, 2, 3\}$ with 64-dim embeddings. Suspecting that degree information may not play a salient role in the sub-graph patterns, we also evaluate the setting without degree information but found this slightly degrades performance compared to models that encode degree attributes. All models use $L = 6$.

\emph{---CIFAR10 and MolPCBA.} We evaluate node-centric \textsc{Elene-L} (\textbf{ND}) with $k \in \{1, 2, 3\}$. Due to computational constraints, we prioritize training with $k=3$ given promising results in other tasks. On CIFAR, we discard uninformative degree information as graphs are $k=8$ nearest neighbor graphs containing super-pixel information. We do not modify the architecture or hyper-parameters of the best-performing GNN-AK\textsuperscript{+} model reported in~\cite{zhao2022}. Our results report average and standard deviations of the evaluation metric---Accuracy for CIFAR10, Average Precision (AP) for MolPCBA---collected from 3 independent runs.\\

\textbf{$h$-Proximity.} See \texttt{proximityResults.sh} for details.\\

We evaluate node-centric \textsc{Elene-L} without degree information, which matches the configuration of SPNNs. We do not tune any hyper-parameters, evaluating \textsc{Elene-L} with $k \in \{3, 5\}$ fixing $L = 3$ and using 32-dim. embeddings. The first layer in the network embeds the color information, for which the model needs to appropriately learn to ignore irrelevant colors. Due to constrained computational resources, we only evaluate two maximum distances for \textsc{Elene-L}, 3 and 5, sharing embedding weights and introducing ego-network signals before each of the 3 GIN layers. We share \textsc{Elene-L} embedding matrices across all layers and set the maximum encoded degree $d = 0$ to only encode distance information. We report the mean and standard deviation of the binary classification accuracy computed across 10-folds over the dataset, following~\cite{abboud2022spnn}.

\textbf{Memory Scalability.} 

We provide additional results from the memory scalability experiments in~\autoref{subsec:ELENEMemScalability}, reporting memory consumption performance when $d_\texttt{max} = 6$ and $d_\texttt{max} = 18$ in~\autoref{fig:ELENEScalabilityPlotDeg6} and~\autoref{fig:ELENEScalabilityPlotDeg18}.

\begin{figure}[!ht]
  \centering
  \includegraphics[width=1.0\linewidth]{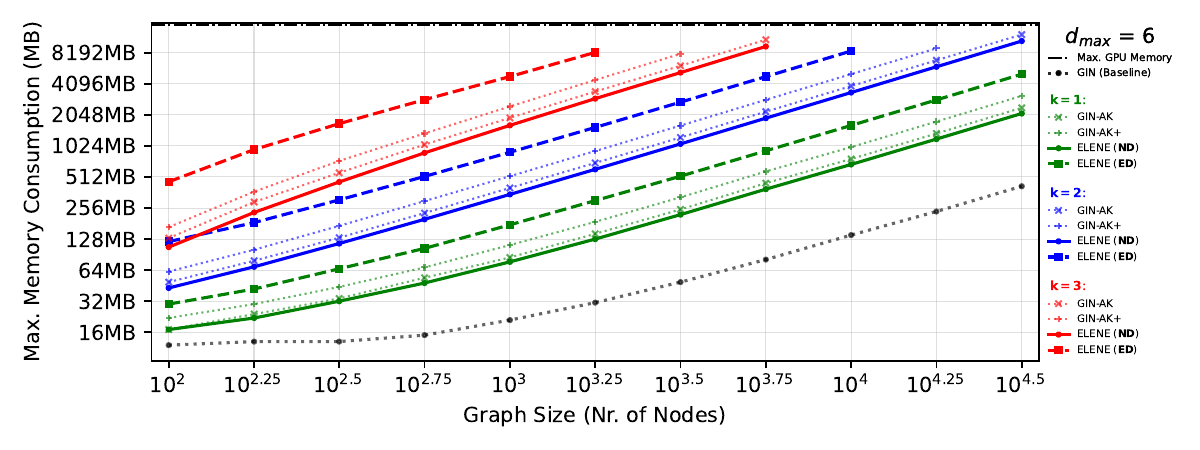}
  \caption[Memory scalability analysis of \textsc{Elene} with $d_\texttt{max} = 6$.]{Memory scalability analysis of \textsc{Elene} with $d_\texttt{max} = 6$, produced as~\autoref{fig:ELENEScalabilityPlot}. We include GIN (dotted line) and maximum GPU memory (dash-and-dotted line) as indicative lower and upper memory bounds. \textsc{Elene-L} (\textbf{ND}, full lines) outperforms both GIN-AK, GIN-AK$^{+}$ (dotted lines) and \mbox{\textsc{Elene-L}} (\textbf{ED}, dashed lines). Additionally, \textsc{Elene-L} (\textbf{ND}) can encode all $d$-regular graphs in the benchmark when $k=1$. As expected, memory consumption increases linearly with the number of nodes as $d_\texttt{max}$ is kept fixed.}
  \label{fig:ELENEScalabilityPlotDeg6}
\end{figure}

\begin{figure}[!ht]
  \centering
  \includegraphics[width=1.0\linewidth]{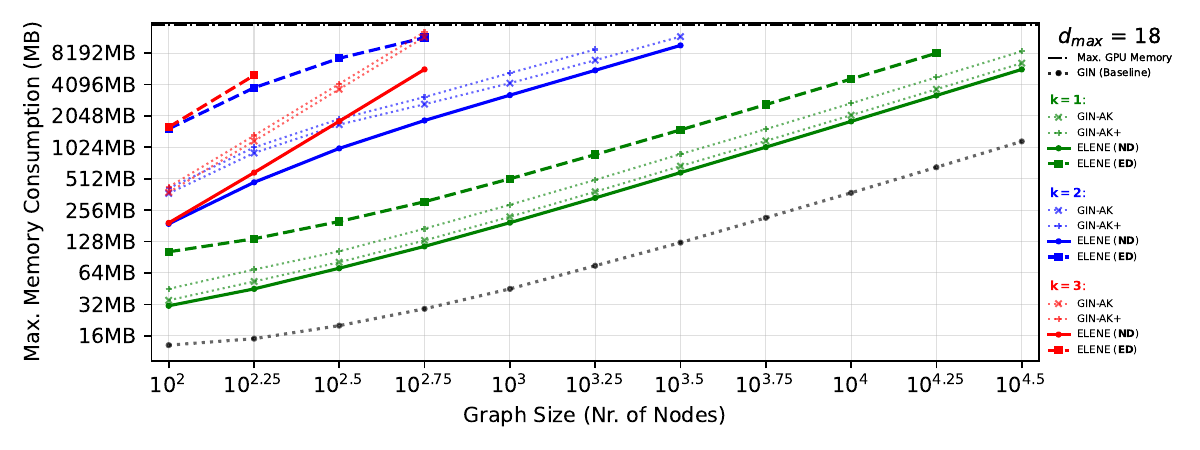}
  \caption{Memory scalability analysis of \textsc{Elene} with $d_\texttt{max} = 18$. See caption of Figure~\ref{fig:ELENEScalabilityPlotDeg6} for details.}
  \label{fig:ELENEScalabilityPlotDeg18}
\end{figure}

\textbf{Graph Density and Scalability.} 

We also provide extended memory scalability results by studying the impact of the density of the graph. We ran additional experiments on graphs where $N=1000$, and evaluated the memory consumption as density increases as a function of the degree of nodes in the graph. We perform the same experiment in two settings: one where the degree distribution is regular (i.e., the graphs are $d$-regular, studying different values of $d$), and one where the distribution of degrees is irregular. In the irregular case, we study the case in which all nodes have \emph{at least} degree $d$, but may have higher connectivity following the Barabási-Albert preferential attachment model~\citep{barabasi-albert1999prefattach}.

\emph{--- Memory Consumption on Regular Density Graphs.} In~\autoref{fig:ELENEDegreeScalabilityPlot}, we compare GIN, GIN-AK, GIN-AK$^{+}$ and ELENE variants on  at depths $k \in \{1, 2, 3\}$. Note that we could not include SPEN, as described in~\autoref{subsec:ELENEMemScalability}, due to reaching the maximum memory thresholds at $d_{\max} = 8$.  

\begin{figure}[!ht]
  \centering
  \includegraphics[width=1.0\linewidth]{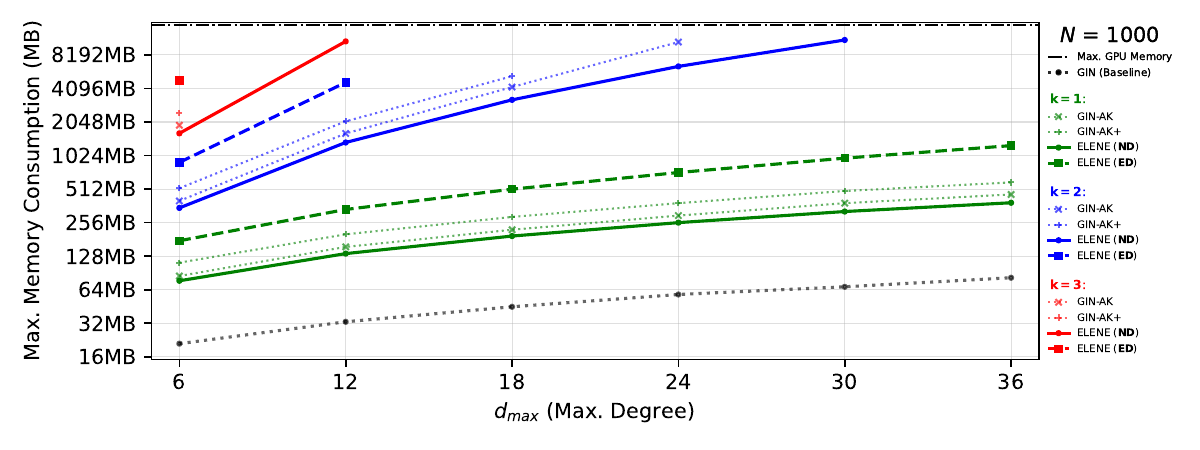}
  \caption{Memory scalability analysis of \textsc{Elene} with $N = 1000$ in function of increasing values of $d_{\max}$. See caption of Figure~\ref{fig:ELENEScalabilityPlotDeg6} for details.}
  \label{fig:ELENEDegreeScalabilityPlot}
\end{figure}

\emph{--- Memory Consumption on Irregular Density Graphs.} In~\autoref{fig:ELENEIrregularDegreeScalabilityPlot}, we repeat the analysis from~\autoref{fig:ELENEDegreeScalabilityPlot} on graphs generated following the preferential attachment model where each node has at least $m$ edges. We find that \textsc{Elene-L} (\textbf{ND}, full lines) outperforms both GIN-AK, GIN-AK$^{+}$ (dotted lines) and \mbox{\textsc{Elene-L}} (\textbf{ED}, dashed lines), matching~\autoref{subsec:ELENEMemScalability} and results on regular connectivity patterns shown in~\autoref{fig:ELENEDegreeScalabilityPlot}.

\begin{figure}[!ht]
  \centering
  \includegraphics[width=1.0\linewidth]{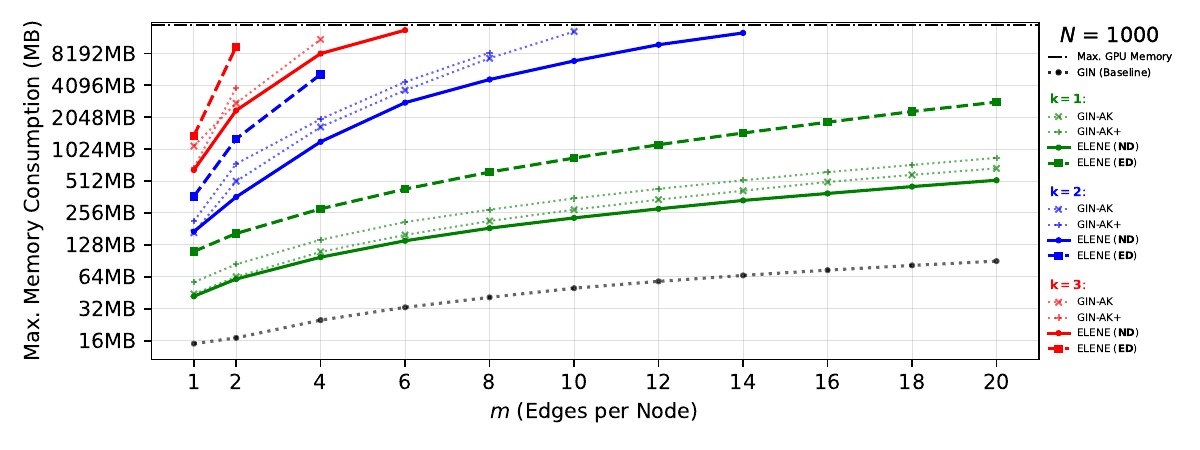}
  \caption{Memory scalability analysis of \textsc{Elene} with $N = 1000$ in function of increasing values of $d_{\min}$ on random Barabási-Albert graphs. See caption of Figure~\ref{fig:ELENEScalabilityPlotDeg6} for details.} \label{fig:ELENEIrregularDegreeScalabilityPlot}
\end{figure}

\subsection{Best Hyper-parameters}

In this section, we provide an overview of the best hyperparameters we find for \textsc{Elene} and \textsc{Elene-L}. For simplicity, we only report the best performing model, i.e., not distinguishing between enhancing a GIN or a GIN-AK\textsuperscript{+} model. We group together hyper-parameters set at the dataset level (e.g. for the Counting Substructures or $h$-Proximity datasets), and report the hyper-parameters corresponding to the best models reported in~\autoref{sec:EleneExperiments}. In our summary, we include the best-performing ego-network feature with (a) the ego-network depth --- $k$, (b) the number of layers --- $L$, and (c) the embedding layer size for \textsc{Elene-L}. 

We summarise our findings in~\autoref{tab:BestHyperparams}. For datasets and tasks where multiple models achieve comparable performance (i.e. same performance metric with the reported significant digits), we break ties by reporting the model with the lowest memory footprint across the tie.

\begin{table}[!t]
\centering
\caption[Best Hyper-parameters on \textsc{Elene}.]{Best hyper-parameters for the for the top performing models after introducing explicit ego-network attributes as shown in~\autoref{sec:EleneExperiments}. We report the hyper-parameters corresponding to the best-performing model by looking at the objective performance metric on each dataset, and resolve ties by selecting the model with the lowest memory footprint.}
\label{tab:BestHyperparams}
\scriptsize
\begin{tabular}{@{}ccccrrrc@{}}
\toprule
\textbf{Benchmark}                                                                    & \textbf{Dataset}                                                                  & \textbf{Task}        & \textbf{Ego-Net Feature}                                                                     & \multicolumn{1}{c}{\textbf{$k$-hops}} & \multicolumn{1}{c}{\textbf{$L$-Layers}} & \multicolumn{1}{c}{\textbf{\begin{tabular}[c]{@{}c@{}}Emb. Size\\ (ELENE-L)\end{tabular}}}               \\ \midrule
\multirow{9}{*}{\textbf{Expr.}}                                                       & \multicolumn{2}{c}{\textbf{EXP}}                                                                         & \textbf{\begin{tabular}[c]{@{}c@{}}All Ego-Net Features\\ Reach 100\% Accuracy\end{tabular}} & 1                                     & 4                                       & 32                                                                                                  \\
                                                                                      & \multicolumn{2}{c}{\textbf{SR25}}                                                                        & \textbf{GIN +ELENE-L (ED)}                                                    & 1                                     & 2                                       & 32                                                                                                  \\ \cmidrule(l){2-8} 
                                                                                      & \multirow{4}{*}{\textbf{\begin{tabular}[c]{@{}c@{}}Counting\\ Sub.\end{tabular}}} & \textbf{Triangle}    & \textbf{GIN-AK\textsuperscript{+}+ELENE}                                            & \multirow{4}{*}{2}                    & \multirow{4}{*}{3}                      & \multirow{4}{*}{16}                                                                                 \\
                                                                                      &                                                                                   & \textbf{Tailed Tri.} & \textbf{GIN-AK\textsuperscript{+}+ELENE}                                            &                                       &                                         &                                                                                                     &                                                                                                   \\
                                                                                      &                                                                                   & \textbf{Star}        & \textbf{GIN +ELENE-L (ND)}                                                         &                                       &                                         &                                                                                                     &                                                                                                   \\
                                                                                      &                                                                                   & \textbf{4-Cycle}     & \textbf{GIN-AK\textsuperscript{+}+ELENE}                                            &                                       &                                         &                                                                                                     &                                                                                                   \\ \cmidrule(l){2-8} 
                                                                                      & \multirow{3}{*}{\textbf{\begin{tabular}[c]{@{}c@{}}Graph\\ Prop.\end{tabular}}}   & \textbf{IsConn.} & \textbf{GIN-AK\textsuperscript{+}+ELENE-L (ND)}                                      & \multirow{3}{*}{3}                    & \multirow{3}{*}{6}                      & \multirow{3}{*}{16}                                                                                 \\
                                                                                      &                                                                                   & \textbf{Diameter}    & \textbf{GIN-AK\textsuperscript{+}+ELENE-L (ND)}                                      &                                       &                                         &                                                                                                     \\
                                                                                      &                                                                                   & \textbf{Radius}      & \textbf{GIN-AK\textsuperscript{+}+ELENE-L (ND)}                                     &                                       &                                         &                                                                                                     \\ \midrule
\multirow{5}{*}{\textbf{\begin{tabular}[c]{@{}c@{}}Real World\\ Graphs\end{tabular}}} & \multicolumn{2}{c}{\textbf{ZINC-12K}}                                                                    & \textbf{GIN+ELENE-L (ND)}                                                           & 3                                     & 6                                       & 32                                                                                                  \\
                                                                                      & \multicolumn{2}{c}{\textbf{CIFAR10}}                                                                     & \textbf{GIN+ELENE-L (ND)}                                                           & 2                                     & 4                                       & 64                                                                                                  \\
                                                                                      & \multicolumn{2}{c}{\textbf{PATTERN}}                                                                     & \textbf{GIN+ELENE-L (ND)}                                                           & 2                                     & 6                                       & 64                                                                                                  \\
                                                                                      & \multicolumn{2}{c}{\textbf{MolHIV}}                                                                      & \textbf{GIN+ELENE}                                                                   & 2                                     & 2                                       & N/A                                                                                  \\
                                                                                      & \multicolumn{2}{c}{\textbf{MolPCBA}}                                                                     & \textbf{GIN+ELENE-L (ND)}                                                           & 3                                     & 5                                       & 64                                                                                                  \\ \midrule
\multirow{4}{*}{\textbf{Proximity}}                                                   & \multirow{4}{*}{\textbf{$h$-Proximity}}                                           & \textbf{$h = 3$}     & \multirow{4}{*}{\textbf{GIN + ELENE-L (ND)}}                                        & 3                                     & \multirow{4}{*}{3}                      & \multirow{4}{*}{32}                                                                                 \\
                                                                                      &                                                                                   & $h = 5$              &                                                                                              & 5                                     &                                         &                                                                                                     &                                                                                                   \\
                                                                                      &                                                                                   & $h = 8$              &                                                                                              & 5                                     &                                         &                                                                                                     &                                                                                                   \\
                                                                                      &                                                                                   & $h = 10$             &                                                                                              & 5                                     &                                         &                                                                                                     &                                                                                                   \\ \bottomrule
\end{tabular}
\end{table}

\section{ELENE is Permutation Equivariant and Invariant}\label{sec:ELENEAppPermInvariant}

We show that \textsc{Elene} is permutation equivariant at the graph level, and permutation invariant at the node level. As all operations that \textsc{Elene} requires are permutation equivariant at the graph level, and permutation invariant at the node level, the same holds for \textsc{Elene} representations.

\begin{lemma}\label{lemma:ELENEPermInvariant}
Given any $v \in V$ for $G = (V, E)$ and given a permuted graph $G' = (V', E')$ of $G$ produced by a permutation of node labels $\pi: V \rightarrow V'$ such that $\forall v \in V \Leftrightarrow \pi(v) \in V'$, $\forall (u, v) \in E \Leftrightarrow (\pi(u), \pi(v)) \in E'$.

All $\textsc{Elene}$ representations are permutation equivariant at the graph level:

$$\pi (\ldblbrace e_{v_1}^{k},\dots ,e_{v_n}^{k} \rdblbrace) = \ldblbrace e_{\pi(v_1)}^{k},\dots ,e_{\pi(v_n)}^{k}\rdblbrace.$$

Furthermore, $\textsc{Elene}$ representations are permutation invariant at the node level:

  $$e_{v}^{k} = e_{\pi(v)}^{k} , \forall v \in V, \pi(v) \in V'.$$
\end{lemma}

\begin{proof}\label{proof:ELENEPermEquivariant}
Note that $e_{v}^{k}$ in~\autoref{eq:ELENEEncRaw} can be expressed in terms of $d_{G}^{(p)}(u|v)$ and $l_{G}(u, v)$. Both $l_G(\cdot,\cdot)$ and $d_G^{(p)}(\cdot|\cdot)$ are permutation invariant functions at the node level and equivariant at the graph level, as they rely on the distance between nodes, which will not change when permutation $\pi(\cdot)$ is applied. Thus, $\textsc{Elene}$ representations are permutation equivariant at the graph level, and permutation invariant at the node level.
\end{proof}

\end{document}